%% file: main.tex
\documentclass{article}

\PassOptionsToPackage{numbers, compress}{natbib}

\usepackage[main, final]{neurips_2026}

\usepackage[utf8]{inputenc} 
\usepackage[T1]{fontenc}    
\usepackage{hyperref}       
\usepackage{url}            
\usepackage{booktabs}       
\usepackage{amsfonts}       
\usepackage{nicefrac}       
\usepackage{microtype}      
\usepackage{xcolor}         

\usepackage{graphicx, wrapfig} 
\usepackage{caption}
\usepackage{subcaption}
\usepackage{makecell}

\usepackage{algorithm}
\usepackage{algpseudocode}
\newcommand{\ie}{\textit{i.e., }}
\newcommand{\eg}{\textit{e.g., }}

\input{math_commands.tex}

\usepackage{array}
\usepackage{xcolor}
\usepackage{subcaption}

\definecolor{cStrawberry}{HTML}{f94144}
\definecolor{cPumpkin}{HTML}{f3722c}
\definecolor{cCarrot}{HTML}{f8961e}
\definecolor{cTangerine}{HTML}{f9844a}
\definecolor{cTuscan}{HTML}{de9f08}
\definecolor{cWillow}{HTML}{6d9f47}
\definecolor{cSeaweed}{HTML}{43aa8b}
\definecolor{cDarkCyan}{HTML}{4d908e}
\definecolor{cBlueSlate}{HTML}{577590}
\definecolor{cCerulean}{HTML}{277da1}

\colorlet{cShared}{cWillow}
\colorlet{cSeparated}{cTuscan}
\colorlet{cIsoEnergy}{cSeaweed}
\colorlet{cGroupSparse}{cBlueSlate}
\colorlet{cPosthocAlign}{cPumpkin}

\setlist[itemize]{leftmargin=1em, itemsep=0.5pt, topsep=0.5pt}

\title{Same Concept, Different Directions:\\ Cross-Modal Feature Heterogeneity\\ in Sparse Autoencoders}

\author{%
\begin{tabular}{c@{\hspace{2em}}c@{\hspace{2em}}c@{\hspace{2em}}c}
\textbf{Chungpa Lee}\thanks{Equal Contribution. Correspondence to: Jy-yong Sohn <jysohn1108@yonsei.ac.kr>.}
&
\textbf{Jihoon Kwon}$^{*}$
&
\textbf{Kyle Min}
&
\textbf{Jy-yong Sohn}
\\
{\normalfont Yonsei University}
&
{\normalfont Seoul National University}
&
{\normalfont Oracle}
&
{\normalfont Yonsei University}
\end{tabular}
}

\begin{document}

\maketitle

\begin{abstract}
Vision--language models map images and text into a joint embedding space. However, these embeddings often entangle multiple semantic features, which limits their interpretability and controllability. While sparse autoencoders have emerged as a useful tool for decomposing these embeddings into monosemantic features, their application to joint embedding spaces has largely relied on an implicit, untested assumption that semantically corresponding features share the same directions across modalities. In this paper, we challenge this assumption by characterizing discrepancies in feature directions for the same concept across image and text modalities, a phenomenon we term \emph{cross-modal feature heterogeneity}. We demonstrate that this heterogeneity is a key driver of the \emph{modality split}, where a shared concept activates different latents depending on the modality. This finding further reveals why aligning latent activations alone is insufficient to resolve the underlying feature mismatch. Motivated by this observation, we propose an approach that trains modality-specific sparse autoencoders to preserve each modality’s feature geometry, and then aligns corresponding features post hoc. Our method improves reconstruction fidelity and enhances performance in cross-modal retrieval and concept-conditioned retrieval.
\end{abstract}

\input{texts/body}

\clearpage

\section*{Acknowledgements}
This work was supported by the National Research Foundation of Korea (NRF) grant funded by the Korea government (MSIT) (No. RS-2024-00345351, No. RS-2024-00408003); by the Institute of Information \& Communications Technology Planning \& Evaluation (IITP) grant funded by the Korea government (MSIT) (No. RS-2025-02283048, Developing the Next-Generation General AI with Reliability, Ethics, and Adaptability); by the IITP under the Leading Generative AI Human Resources Development (IITP-2026-RS-2026-25544647) grant funded by the Korea government (MSIT); and by the National Institute for Mathematical Sciences (NIMS) grant funded by the Korea government (MSIT) (No. S26120000).

\bibliographystyle{plainnat}
\bibliography{__ref}

\clearpage
\appendix
\include{texts/appendix}

\clearpage
\input{texts/checklist}

\end{document}

%% file: math_commands.tex
\usepackage{amsmath, amsthm, amssymb, amsfonts,bm}
\usepackage{mathtools}
\mathtoolsset{showonlyrefs}

\usepackage{thmtools}
\usepackage{thm-restate}

\usepackage[table]{xcolor}

\definecolor{myred}{HTML}{FF6B6B}
\definecolor{myyellow}{HTML}{FFD93D}
\definecolor{mygreen}{HTML}{6BCB77}
\definecolor{myblue}{HTML}{4D96FF}

\usepackage{enumitem}
\setlist[enumerate]{leftmargin=*, label=\roman*.} 
\setlist[itemize]{leftmargin=*, itemsep=0pt, topsep=0pt, partopsep=0pt}

\newtheorem{definition}{Definition}

\newtheorem*{theorem*}{Theorem}
\newtheorem*{proposition*}{Proposition}
\newtheorem*{lemma*}{Lemma}
\newtheorem*{corollary*}{Corollary}
\newtheorem*{example*}{Example}

\def\Pr{\mathop{\rm Pr}\nolimits} 

\def\E{\displaystyle \mathbb{E}}

\def\Corr{\displaystyle\mathrm{Corr}}

\newcommand{\loss}{\mathcal{L}}

\newcommand{\norm}[1]{\left\lVert#1\right\rVert}

\DeclareMathOperator*{\argmax}{arg\,max}

\DeclareMathOperator{\tr}{tr}

\def\vzero{{\mathbf{0}}}

\def\vtheta{{\mathbf{\theta}}}

\def\vu{{\mathbf{u}}}
\def\vv{{\mathbf{v}}}
\def\vw{{\mathbf{w}}}
\def\vx{{\mathbf{x}}}
\def\vy{{\mathbf{y}}}
\def\vz{{\mathbf{z}}}

\def\vC{{\mathbf{C}}}

\def\vM{{\mathbf{M}}}

\def\vP{{\mathbf{P}}}

\def\vV{{\mathbf{V}}}
\def\vW{{\mathbf{W}}}

\def\vZ{{\mathbf{Z}}}

\def\vPhi{{\mathbf{\Phi}}}
\def\vPsi{{\mathbf{\Psi}}}

\def\vphi{{\bm{\phi}}}
\def\vpsi{{\bm{\psi}}}

\def\sA{{\mathbb{A}}}

\def\sR{{\mathbb{R}}}
\def\sS{{\mathbb{S}}}

%% file: texts/body.tex
\section{Introduction}
Vision-language models (VLMs) map images and text into a joint representation space, which supports a wide range of downstream tasks~\citep{radford21clip, yu2022coca, chen2024understanding} and serves as the foundation for generative VLMs~\citep{liu2023visual, rombach2022high}. Despite their empirical success, the joint representation remains difficult to interpret~\citep{oikarinen2023clipdissect}. These embeddings are typically \emph{polysemantic}, as they encode multiple semantic concepts in a single vector, and it remains unclear how each concept is encoded and shared across modalities.

The linear representation hypothesis~\citep{elhage2022toy, park2024lrh} provides a useful framework for studying this question, positing that each embedding can be expressed as a linear combination of a small number of \emph{monosemantic} features. Sparse autoencoders (SAEs) operationalize this hypothesis by mapping each embedding to a \emph{sparse latent code}, whose active coordinates indicate which monosemantic features are present in the embedding. Building on their success in unimodal settings~\citep{huben2023sparse, templeton2024scaling, gao2025scaling, markssparse}, recent work has applied SAEs to joint embedding spaces to recover monosemantic features shared across image and text modalities~\citep{costa2026from, zaigrajew2025interpreting, zhang2025large, papadimitriou2025interpreting, pach2026sparse, kaushik2026learning, dhimoila2026crossmodal}.

When SAEs are applied to joint embeddings of VLMs, prior work~\citep{papadimitriou2025interpreting} observes a phenomenon known as \emph{modality split}, where the same concept activates different latent coordinates across modalities. This split makes latent codes difficult to use across modalities, motivating prior approaches to align latent activations across modalities~\citep{kaushik2026learning, dhimoila2026crossmodal}. However, these approaches treat modality split as a mismatch in latent activations, while implicitly assuming that semantically corresponding features share the same directions across modalities.

\input{figures/problem}

In this paper, we revisit this assumption by asking whether semantically corresponding feature directions are indeed aligned across modalities in the learned embedding space. We characterize \emph{cross-modal feature heterogeneity}, a phenomenon in which the same semantic concept can be represented by different feature directions across image and text modalities, as illustrated in Figure~\ref{fig:problem}.

Accounting for this heterogeneity is crucial as it redefines our interpretation of the modality split. When corresponding image and text features have different directions in the joint embedding space, an SAE naturally assigns them to different latent coordinates to ensure precise reconstruction. Consequently, while forcing these latent activations to align across different modalities might mitigate the modality split, it risks degrading feature recovery by collapsing geometrically distinct directions into a single coordinate. This motivates our approach, which trains modality-specific SAEs to preserve each modality’s feature geometry and then aligns corresponding features post hoc.

Our contributions are summarized below.
\begin{itemize}
    \item In Section~\ref{sec:heterogeneity:definition}, we characterize \emph{cross-modal feature heterogeneity}, where the same concept can have different feature directions across modalities in the joint embedding space.
    \item In Section~\ref{sec:theory}, we show that this heterogeneity can explain \emph{modality split} in multimodal SAEs and analyze how existing alignment approaches trade reconstruction quality for latent alignment.
    \item In Section~\ref{sec:ours}, we propose an approach that preserves the unique feature geometry of each modality and aligns corresponding latent coordinates post hoc without sacrificing reconstruction quality.
    \item In Section~\ref{sec:experiments}, we show that our approach improves reconstruction fidelity and cross-modal retrieval, and enables more effective concept-conditioned retrieval on real VLM embeddings.
\end{itemize}

\section{Related Work}
\label{sec:related:work}

\paragraph{Mechanistic interpretability and linear representation hypothesis.}
Mechanistic interpretability aims to reverse-engineer neural networks, often by interpreting their learned representations as human-understandable mechanisms~\citep{conmy2023towards, elhage2021mathematical, wang2023interpretability, chughtai2023toy}. A key assumption underlying this line of work is the \emph{linear representation hypothesis}~\citep{elhage2022toy, park2024lrh}, which posits that each semantically coherent concept is encoded as a unique direction in representation space, and that embeddings can be expressed as linear combinations of such directions. This hypothesis is supported by diverse empirical evidence from linear probing~\citep{alain2017understanding, hewitt2019structural} and steering interventions~\citep{li2023inference, arditi2024refusal, rimsky2024steering}. Recent work has begun extending this hypothesis to multimodal representations~\citep{patashnik2021style, brack2023sega, liu2025reducing}, opening avenues for understanding feature structure across different modalities.

\paragraph{Sparse autoencoders for interpreting representations.}
SAEs have shown great promise for extracting monosemantic features from embeddings of language models~\citep{huben2023sparse, templeton2024scaling, gao2025scaling, markssparse, harle2025monosemanticity}. Building on these advances, SAEs have been applied to VLMs~\citep{rao2024discover, zaigrajew2025interpreting, pach2026sparse}. Recent work has further examined how these features align or separate across modalities~\citep{costa2026from, zaigrajew2025interpreting, zhang2025large, papadimitriou2025interpreting, pach2026sparse}, motivating an analysis of their underlying structural properties in joint embedding spaces.

\paragraph{Joint representations of vision-language models.}
A prominent phenomenon in joint VLM representations is the \emph{modality gap}, where image and text embeddings occupy disjoint regions of the shared space~\citep{liang2022mind}.
Prior work has analyzed this gap from multiple geometric, distributional, and optimization-level perspectives~\citep{liang2022mind, zhang2024connect, schrodi2025two}.
One line of work argues that the modality gap partly reflects modality-specific features in multimodal representations, rather than merely a failure of alignment~\citep{jiang2023understanding, ramasinghe2024accept, qian2023intra}.
Another line focuses on cross-modal misalignment of shared features and studies how reducing such mismatch improves downstream task performance~\citep{eslami2025mitigate, yamaguchi2025post, grassucci2026closing, mistretta2025cross}.

Recent SAE-based analyses further reveal \emph{modality split}~\citep{papadimitriou2025interpreting}, where the same concept activates different latent coordinates across modalities. Existing remedies force shared latent activations~\citep{kaushik2026learning, dhimoila2026crossmodal}, implicitly assuming feature directions are aligned (\ie $\vphi_i=\vpsi_i$) and overlooking directional misalignment. When shared concepts are encoded along distinct directions, forcing their latent activations to align during training can degrade reconstruction. We instead preserve modality-specific directions by using modality-specific SAEs and align the corresponding features post hoc.

\section{Preliminaries}
\label{sec:prelim}

We formalize joint embeddings from VLMs as linear combinations of monosemantic features. We then introduce SAEs as a framework for recovering feature directions from these embeddings.

\paragraph{Embeddings as linear combinations of features.}
Following the linear representation hypothesis~\citep{arora2018linear, elhage2022toy, park2024lrh, cui2026on}, we assume that VLM embeddings $\vx, \vy \in \mathbb{R}^d$ are polysemantic, meaning that each embedding can be expressed as a linear combination of monosemantic features. A monosemantic feature is a direction in the joint embedding space that represents a semantically coherent concept.

Formally, let $n$ be the number of latent concepts represented in the joint embedding space, and define $[n] := \{1,2,\cdots,n\}$. We assume the existence of a latent code $\vz := (z_1\; \cdots\; z_n)^\top \in \mathbb{R}_+^n$ and \emph{feature matrices} $\vPhi := (\vphi_1\; \cdots\; \vphi_n) \in \mathbb{R}^{d \times n}$ for images and $\vPsi := (\vpsi_1\; \cdots\; \vpsi_n) \in \mathbb{R}^{d \times n}$ for text. The column vectors $\vphi_i, \vpsi_i \in \mathbb{R}^d$ denote the \emph{feature direction vectors} of the $i$-th latent concept for images and text, respectively. We assume that all feature directions have unit norm.

Under the linear representation hypothesis, an image embedding is written as $\vx = \sum_{i\in[n]} z_i\vphi_i$, where each activation value $z_i \in \sR_{+}$ indicates the strength of the corresponding feature, whereas $z_i = 0$ indicates that it is inactive. When an image--text pair is considered, we assume that the two embeddings share the same latent code $\vz$, while their feature directions may differ across modalities:
\begin{align}
    \label{eq:embeddings}
    \vx = \vPhi \vz ={\textstyle \sum_{i\in[n]}} z_i \vphi_i,
    \qquad
    \vy = \vPsi \vz ={\textstyle \sum_{i\in[n]}} z_i \vpsi_i.
\end{align}
Unlike prior work~\citep{kaushik2026learning}, we impose no constraint such as $\vphi_i = \vpsi_i$. Thus, we allow $\vphi_i \neq \vpsi_i$, capturing the possibility that the same concept is encoded along different directions across modalities.

We further assume that the coordinates of the latent code $\vz$ are independent and sparse. Specifically, the sparsity parameter $s$ is defined as the probability that each coordinate is inactive,
\begin{align}
    \label{eq:sparsity}
    \Pr(z_i = 0) = s, \qquad i \in [n],
\end{align}
where $s$ is close to $1$. This implies that only a small number of features are active in each embedding.

\paragraph{Sparse autoencoder.}
SAEs encode each embedding into a sparse latent code and decode this code to reconstruct the original embedding~\citep{huben2023sparse, templeton2024scaling, gao2025scaling, rao2024discover, zaigrajew2025interpreting, cui2026on}. Ideally, each active coordinate of the latent code represents the strength of a monosemantic feature in the embedding, while the corresponding decoder column provides its feature direction.

Specifically, let $m$ be the latent dimension of the SAE, and let $\vW:=[\vw_1 \cdots \vw_m] \in \mathbb{R}^{d \times m}$ denote its weight matrix. For simplicity, we omit bias terms and use the same weight matrix for the encoder and decoder. The encoder produces an \emph{estimated latent code} $\tilde{\vz}:=(\tilde{z}_1\; \cdots\; \tilde{z}_m)^\top \in \mathbb{R}_+^m$. For image and text embeddings, we write
\begin{align}
    \label{eq:encoder}
    \tilde{\vz}(\vx) := \sigma(\vW^\top \vx), \qquad
    \tilde{\vz}(\vy) := \sigma(\vW^\top \vy),
\end{align}
where $\sigma$ denotes an activation function\footnote{For theory, we use the Top-1 operator, whose behavior can be approximated by Top-$K$ activations with bias terms. In experiments, we use SAEs with Top-$K$ activations for $K>1$, bias terms, and separate encoder and decoder weights.} that induces sparsity.
The decoder then produces
\begin{align}
    \label{eq:decoder}
    \tilde{\vx}(\vx) := \vW \tilde{\vz}(\vx)
    ={\textstyle \sum_{j\in[m]}} \tilde{z}_j(\vx)\; \vw_j, \qquad
    \tilde{\vy}(\vy) := \vW \tilde{\vz}(\vy)
    ={\textstyle \sum_{j\in[m]}} \tilde{z}_j(\vy)\; \vw_j.
\end{align}

The SAE is trained to reconstruct embeddings from both modalities by minimizing
\begin{align}
    \label{eq:loss}
    \loss_{\mathrm{rec}}(\vW; \vPhi) + \loss_{\mathrm{rec}}(\vW; \vPsi)
    :=
    \mathbb{E}_{\vx}\!\norm{\vx-\tilde{\vx}(\vx)}_2^2
    + \mathbb{E}_{\vy}\!\norm{\vy-\tilde{\vy}(\vy)}_2^2,
\end{align}
where the first term denotes the reconstruction loss for image embeddings generated by $\vPhi$, and the second term denotes the reconstruction loss for text embeddings generated by $\vPsi$, as specified in~\eqref{eq:embeddings}.
More explicitly, $\loss_{\mathrm{rec}}(\vW; \vPhi) :=
 \mathbb{E}_{\vx}\!\norm{\vx-\tilde{\vx}(\vx)}_2^2
 = \mathbb{E}_{\vz}\!\norm{\vPhi \vz - \vW \sigma\!\big( \vW^\top\vPhi\vz \big)}_2^2$.
After training, the encoder extracts latent codes, and the decoder columns provide estimates of feature directions~\citep{cui2026on}.

\paragraph{Concept alignment across modalities in sparse autoencoders.}
When an SAE is trained on embeddings from both modalities, it is expected to \emph{recover monosemantic features} within each modality and to \emph{align shared concepts} by assigning corresponding image and text features to shared latent coordinates. However, each latent index $j$ is tied to a single column $\vw_j$ of the decoder weight matrix in \eqref{eq:decoder}. Assigning corresponding features to the same coordinate implicitly assumes that they share a common direction in the joint embedding space. We first examine whether this assumption holds in learned joint embedding spaces.

\section{Cross-Modal Feature Heterogeneity in Joint Embedding Spaces}
\label{sec:heterogeneity:definition}

We examine whether semantically corresponding features share a common direction across image and text modalities in the joint embedding space. When they do not, we refer to this directional mismatch as \emph{cross-modal feature heterogeneity}, which we formally define below.

\begin{definition}[Cross-Modal Feature Heterogeneity]
\label{def:heterogeneity}
The $i$-th latent concept exhibits cross-modal feature heterogeneity if its image and text feature directions are not perfectly aligned, \ie $\vphi_i \neq \vpsi_i$.
\end{definition}

This definition separates semantic correspondence from directional identity in the joint embedding space. In other words, an image feature and a text feature may represent the same concept while occupying different directions in the joint embedding space. We next examine whether such directional differences appear in joint embedding spaces learned by VLMs.

\paragraph{Measuring differences in feature directions across modalities.}
Since the true feature directions are unobserved, we use decoder columns as their estimates~\citep{cui2026on}, letting $\hat{\vphi}_i$ and $\hat{\vpsi}_j$ denote the $i$-th and $j$-th decoder columns of SAEs trained on image and text embeddings, respectively.

The correspondences between image and text features are also unknown. Following prior work~\citep{fel2025archetypal,papadimitriou2025interpreting,spielman2012exact}, we characterize corresponding features across independently trained SAEs, using coactivation correlations on paired image and text embeddings as a proxy for semantic correspondence. Using the encoders of the SAEs trained on image and text embeddings, we compute latent codes $\tilde{\vz}(\vx)$ and $\tilde{\vz}(\vy)$ for paired embeddings $(\vx,\vy)$ and form their correlation matrix over the training set:
\begin{align}
    \label{eq:correlation}
    \vC := \Corr\big(\tilde{\vz}(\vx), \tilde{\vz}(\vy)\big) \in \mathbb{R}^{m \times m},
    \qquad c_{i,j} := [\vC]_{i,j}.
\end{align}
Each entry $c_{i,j}$ measures how strongly the $i$-th image latent and the $j$-th text latent coactivate on paired inputs. Larger values are treated as indicating more likely semantic correspondence. To test whether such likely corresponding features are also directionally aligned in the joint embedding space, we measure the distance between their estimated feature directions. Specifically, for each pair $(i,j)$, we compute the cosine distance between $\hat{\vphi}_i$ and $\hat{\vpsi}_j$, defined as $d_{\cos}(\hat{\vphi}_i, \hat{\vpsi}_j) := 1 - \cos(\hat{\vphi}_i, \hat{\vpsi}_j) \in [0,2]$.

\input{figures/multi_density}

Figure~\ref{fig:multi_density} reports the distribution of cosine distances $d_{\cos}$ between estimated image and text feature directions, grouped by coactivation correlation $c$ in~\eqref{eq:correlation}. The results are computed on \texttt{MS-COCO}~\citep{lin2014microsoft} embeddings extracted from four VLMs, including \texttt{CLIP}~\citep{radford21clip}, \texttt{MetaCLIP}~\citep{chuang2026meta}, \texttt{OpenCLIP}~\citep{cherti2023reproducible}, and \texttt{SigLIP2}~\citep{tschannen2025siglip}. Feature pairs with larger coactivation correlation tend to have smaller distances, but they do not concentrate near zero distance. This provides empirical support for cross-modal feature heterogeneity, where semantically corresponding image and text features need not share identical directions in the joint embedding space. See Appendix~\ref{sec:exp:detail:sec3} for experimental details.

\paragraph{Validating the observed cross-modal feature heterogeneity.}
Beyond identifying corresponding features using coactivation correlations on paired image and text embeddings from \texttt{MS-COCO}, we validate cross-modal feature heterogeneity by using human-annotated object categories. \texttt{MS-COCO} provides object category annotations for each image. We select 65 categories with at least 50 images containing the category and 50 captions mentioning it, allowing us to examine semantic correspondence independently of coactivation.

For each selected category, we independently identify a latent coordinate in each modality whose activations distinguish samples associated with that category from those not associated with it. Based on the image and text coordinates identified for the same category, we compute the cosine distance between the feature directions given by the corresponding decoder columns. The mean distance across the selected categories is $0.511$, showing directional differences between image and text features of the same human-annotated category. See Appendices~\ref{app:annotation_detail} and~\ref{app:semantic_validation} for details and results.

In Appendix~\ref{app:seed_stability}, we assess cross-modal feature heterogeneity by comparing feature directions within and across modalities, showing that image--text feature pairs have larger distances than image--image or text--text pairs from independently trained SAEs. Moreover, Appendix~\ref{app:qualitative_examples} provides qualitative examples of semantically corresponding features with different directions.

Appendix~\ref{app:heterogeneity} discusses possible sources of this heterogeneity and relates it to the modality gap, showing that a nontrivial modality gap implies cross-modal feature heterogeneity.

\section{Analysis of Sparse Autoencoders under Cross-Modal Feature Heterogeneity}
\label{sec:theory}
We examine how cross-modal feature heterogeneity affects feature recovery by SAEs through reconstruction and latent alignment. We investigate whether a shared SAE can recover image and text feature directions while assigning semantically corresponding features to the same latent coordinates. Under cross-modal feature heterogeneity, these two goals can conflict. Corollary~\ref{thm:m-large} explains why reconstruction alone induces modality split, Theorem~\ref{thm:m-small} shows how limited latent capacity collapses distinct feature directions, and Propositions~\ref{thm:prev:group} and~\ref{thm:prev:align} show why existing alignment losses reduce modality split only by sacrificing reconstruction quality. Proofs are provided in Appendix~\ref{sec:proofs}.

\paragraph{Why reconstruction induces modality split.}
We consider an SAE trained with the reconstruction loss in~\eqref{eq:loss}. Since no alignment term is used, the objective only encourages accurate reconstruction.
\begin{restatable}{corollary}{thmlarge}\label{thm:m-large}
    Suppose the SAE in~\eqref{eq:encoder} and~\eqref{eq:decoder} has $m \geq 2n$ latent coordinates. Consider the loss
    $\loss_{\mathrm{rec}}(\vW; \vPhi)+\loss_{\mathrm{rec}}(\vW; \vPsi)$ in~\eqref{eq:loss}.
    As the sparsity parameter $s$ in~\eqref{eq:sparsity} approaches $1$, the loss is minimized to leading order by taking
    $\hat{\vW} := [\vPhi \;\; \vPsi \;\; \vzero_{d \times (m-2n)}]\,\vP$ for any permutation matrix $\vP \in \sR^{m \times m}$. Moreover, $\loss_{\mathrm{rec}}(\hat{\vW};\, \vPhi) + \loss_{\mathrm{rec}}(\hat{\vW};\, \vPsi) = o(1-s)$.
\end{restatable}
Corollary~\ref{thm:m-large} shows that when the SAE has enough latent coordinates, the reconstruction loss is minimized to leading order by assigning a separate column of the decoder weight matrix $\vW$ to each feature direction. Thus, even if the image feature $\vphi_i$ and the text feature $\vpsi_i$ represent the same concept, the reconstruction objective does not force them to share a latent coordinate when they point in different directions. Instead, the SAE can reconstruct them by assigning them to different columns, say $\hat{\vw}_a$ and $\hat{\vw}_b$, which activate different latent coordinates, $\tilde{z}_a$ and $\tilde{z}_b$. This provides a mechanism for modality split: the same semantic concept can appear at different latent indices across modalities even when both feature directions are accurately recovered.

\paragraph{Why limited capacity collapses distinct features.}
Corollary~\ref{thm:m-large} assumes enough latent coordinates to represent all feature directions separately, namely $m \geq 2n$. In practice, the effective number of coordinates can be smaller due to dead neurons~\citep{templeton2024scaling, gao2025scaling, dylingrelu2020}, \ie $\tilde{z}_j(\vx)=0$ for all $\vx$. We therefore analyze the regime where the SAE has fewer latent coordinates than feature directions, $m < 2n$.
\begin{restatable}{theorem}{thmsmall}
    \label{thm:m-small}
    Suppose the SAE in~\eqref{eq:encoder} and~\eqref{eq:decoder} has $m < 2n$ latent coordinates. Consider the loss
    $\loss_{\mathrm{rec}}(\vW; \vPhi)+\loss_{\mathrm{rec}}(\vW; \vPsi)$ in~\eqref{eq:loss}.
    Define $\vM_i := \E\big[z_i^2 \mid z_i \neq 0\big]\,\vphi_i\vphi_i^\top \in \sR^{d \times d}$
    and $\vM_{n+i} := \E\big[z_i^2 \mid z_i \neq 0\big]\,\vpsi_i\vpsi_i^\top \in \sR^{d \times d}$ for $i \in [n]$.
    Let $(\sA_1, \cdots, \sA_m)$ be a partition of $[2n]$ that maximizes
    $\sum_{j \in [m]} \lambda_{\max}\!\left(\sum_{i \in \sA_j} \vM_i\right)$,
    where $\lambda_{\max}(\cdot)$ denotes the largest eigenvalue.
    Assume that each matrix $\sum_{i\in\sA_j}\vM_i$ has a unit-norm top eigenvector $\hat{\vw}_j$ satisfying $\hat{\vw}_j^\top\vphi_i\ge\nobreak0$ for $i\in\sA_j\cap[n]$ and $\hat{\vw}_j^\top\vpsi_i\ge\nobreak0$ for $i\in[n]$ with $n+i\in\sA_j$.
    As the sparsity parameter $s$ in~\eqref{eq:sparsity} approaches $1$, the loss is minimized to leading order by taking $\hat{\vW}:=[\hat{\vw}_1\;\cdots\;\hat{\vw}_m]$.
    Moreover, if no two directions among $\{\vphi_i\}_{i\in[n]}\cup\{\vpsi_i\}_{i\in[n]}$ are collinear, $\loss_{\mathrm{rec}}(\hat{\vW};\vPhi) + \loss_{\mathrm{rec}}(\hat{\vW};\vPsi) >0$ for $s$ sufficiently close to $1$.
\end{restatable}
Theorem~\ref{thm:m-small} analyzes the regime $m<2n$, where the SAE has fewer latent coordinates than the total number of image and text feature directions. In this regime, some distinct feature directions must share a latent coordinate. The partition $(\sA_1,\ldots,\sA_m)$ describes this assignment over the set of all image and text feature directions. More precisely, all feature directions whose indices belong to the same set $\sA_j$ are represented by the same column $\hat{\vw}_j$ and therefore activate the same coordinate $\tilde{z}_j$.

For each group $\sA_j$, the column $\hat{\vw}_j$ is chosen as the best single direction for representing the directions in that group, namely the unit-norm top eigenvector of $\sum_{i\in\sA_j}\vM_i$, where each $\vM_i$ is the weighted outer product of the corresponding feature direction. Thus, the partition is determined by which feature directions can be well represented by a common decoder column. With equal weights $\E[z_i^2\mid z_i\neq 0]$, feature directions that are close in angle tend to be grouped together.

Consequently, two feature directions that are close in the embedding space may be represented by the same column $\hat{\vw}_j$ and activate the same latent coordinate $\tilde{z}_j$, even if they correspond to different concepts. In a favorable case, if $\vphi_i$ and $\vpsi_i$ represent the same concept and are also close in direction, limited capacity may group them into the same coordinate, reducing modality split. However, as discussed in Section~\ref{sec:heterogeneity:definition}, semantic correspondence does not guarantee directional alignment. Therefore, limited capacity alone does not provide a reliable solution to modality split, and an additional alignment procedure is needed to identify corresponding concepts across modalities.

\paragraph{Why alignment during training sacrifices reconstruction.}
The previous results suggest that reducing modality split requires placing corresponding image and text features on a shared latent coordinate. Existing methods encourage this behavior by adding an auxiliary alignment loss during training~\citep{dhimoila2026crossmodal, kaushik2026learning}. We analyze how such losses improve alignment at the cost of reconstruction quality.

We characterize global minimizers $\hat{\vW}$ for two alignment objectives, the group-sparse loss~\citep{kaushik2026learning} and the Iso-Energy alignment loss~\citep{dhimoila2026crossmodal}. Each objective adds an auxiliary term weighted by $\lambda$, which controls the strength of alignment. Because the objectives are symmetric in the image and text modalities, the following propositions report only $\loss_{\mathrm{rec}}(\hat{\vW};\vphi)$, with the same value holding for $\loss_{\mathrm{rec}}(\hat{\vW};\vpsi)$. In both cases, the form of the global minimizer changes at critical values of $\lambda$.
\begin{restatable}{proposition}{thmgroup}
    \label{thm:prev:group}
    Consider the case $n=1$, where the feature directions $\vphi,\vpsi \in \sR^d$ satisfy $\rho := \vphi^\top\vpsi \in (0,1)$. Suppose the SAE in~\eqref{eq:encoder} and~\eqref{eq:decoder} has $m \geq 2$ latent coordinates. Consider the loss~\citep{kaushik2026learning}
    \begin{align}
        \loss_{\mathrm{rec}}(\vW;\vphi)
        +
        \loss_{\mathrm{rec}}(\vW;\vpsi)
        +
        \lambda\,
        \E_z\!\bigg[
        {\textstyle \sum_{j \in [m]}}
        \sqrt{\big[\sigma\!\big(\vW^\top\vphi\, z\big)\big]_j^2
              + \big[\sigma\!\big(\vW^\top\vpsi\, z\big)\big]_j^2}
        \;\bigg].
    \end{align}
    Define
    $\lambda^\star(\rho) \!:=\! \tfrac{(1-\rho)}{2-\sqrt{1+\rho}} \cdot \tfrac{\E[z^2]}{\E[z]}$
    and $\lambda^{\star\star}(\rho) \!:=\! \sqrt{1+\rho}\cdot \tfrac{\E[z^2]}{\E[z]}$.
    Up to any permutation matrix $\vP$, a global minimizer $\hat{\vW}$ of the loss is:
    \begin{itemize}
        \item If $\lambda < \lambda^\star(\rho)$, then $\hat{\vW} = [\vphi\;\vpsi\;\vzero_{d\times(m-2)}]\,\vP$, where $\loss_{\mathrm{rec}}(\hat{\vW};\vphi) = 0$.
        \item If $\lambda \in \bigl(\lambda^\star(\rho),\,\lambda^{\star\star}(\rho)\bigr)$, then $\hat{\vW} = \big[\tfrac{\vphi+\vpsi}{\norm{\vphi+\vpsi}}\;\vzero_{d\times(m-1)}\big]\,\vP$, where $\loss_{\mathrm{rec}}(\hat{\vW};\vphi) = \tfrac{1-\rho}{2}\,\E[z^2]$.
        \item If $\lambda > \lambda^{\star\star}(\rho)$, then $\hat{\vW}$ satisfies $\sigma(\hat{\vW}^\top\vphi) = \sigma(\hat{\vW}^\top\vpsi) = \vzero$, where $\loss_{\mathrm{rec}}(\hat{\vW};\vphi)  = \E[z^2]$.
  \end{itemize}
\end{restatable}
\begin{restatable}{proposition}{thmalign}
    \label{thm:prev:align}
    Consider the case $n=1$, where the feature directions $\vphi,\vpsi \in \sR^d$ satisfy $\rho := \vphi^\top\vpsi \in (0,1)$. Suppose the SAE in~\eqref{eq:encoder} and~\eqref{eq:decoder} has $m \geq 2$ latent coordinates.
    Consider the loss~\citep{dhimoila2026crossmodal}
    \begin{align}
        \loss_{\mathrm{rec}}(\vW;\vphi)
        +
        \loss_{\mathrm{rec}}(\vW;\vpsi)
        -
        \lambda\,
        \E_{z}\!\Big[
        \sigma\!\big(\vW^\top\vphi\, z\big)^\top
        \sigma\!\big(\vW^\top\vpsi\, z\big)
        \Big].
    \end{align}
    Define $\lambda^\star(\rho) := \tfrac{2(1-\rho)}{1+\rho}$.
    Up to any permutation matrix $\vP$, a global minimizer $\hat{\vW}$ of the loss is:
    \begin{itemize}
        \item If $\lambda < \lambda^\star(\rho)$, then $\hat{\vW} = [\vphi\;\vpsi\;\vzero_{d\times(m-2)}]\,\vP$, where $\loss_{\mathrm{rec}}(\hat{\vW};\vphi)  = 0$.
        \item If $\lambda > \lambda^\star(\rho)$, then $\hat{\vW} = \big[\tfrac{\vphi+\vpsi}{\norm{\vphi+\vpsi}} \; \vzero_{d\times(m-1)}\big]\vP$, where $\loss_{\mathrm{rec}}(\hat{\vW};\vphi) = \tfrac{1-\rho}{2}\,\E[z^2]$.
  \end{itemize}
\end{restatable}
Propositions~\ref{thm:prev:group} and~\ref{thm:prev:align} illustrate the trade-off between reconstruction and latent alignment. For both losses, a small $\lambda$ keeps the image and text feature directions in two separate decoder columns, $\vphi$ and $\vpsi$, so they remain on distinct latent coordinates. Once $\lambda$ exceeds $\lambda^\star(\rho)$, the two directions are represented by a single shared column, $\frac{\vphi+\vpsi}{\norm{\vphi+\vpsi}}$, placing both on one coordinate. This enhances latent alignment but reduces reconstruction quality, with $\loss_{\mathrm{rec}}(\hat{\vW};\vphi)=\tfrac{1-\rho}{2}\E[z^2]$. For the group-sparse loss, an even larger $\lambda$ beyond $\lambda^{\star\star}(\rho)$ drives the SAE to a degenerate solution where both features become inactive. This trade-off motivates learning image and text feature directions separately to avoid competition for the same decoder columns. The recovered latent coordinates can then be matched after training without changing the learned directions.

\section{Modality-Specific Sparse Autoencoders and Post-Hoc Alignment}
\label{sec:ours}

\input{figures/method}

Motivated by the analysis in Section~\ref{sec:theory}, we propose a two-stage method. First, we train separate SAEs for image and text embeddings to preserve modality-specific feature directions. Then, we align their latent coordinates using activation correlations. Figure~\ref{fig:method} provides an overview of this procedure.

\paragraph{Modality-specific sparse autoencoders.}
We train separate SAEs for image and text embeddings, with weight matrices $\vV$ and $\vW$, respectively, using the reconstruction loss in~\eqref{eq:loss}. We minimize $\loss_{\mathrm{rec}}(\vV; \vPhi)$ and $\loss_{\mathrm{rec}}(\vW; \vPsi)$ separately, without any cross-modal alignment constraint. This prevents image and text feature directions from competing for the same columns of the weight matrix.

\paragraph{Post-hoc alignment.}
Because the latent coordinates of the two SAEs are not naturally aligned, we use the correlation matrix $\vC$ in~\eqref{eq:correlation} to identify corresponding image and text latents.
Specifically, we find a permutation matrix $\hat{\vP}$ that maximizes the total correlation between aligned coordinates\footnote{Dead latent coordinates are excluded in practice. Thus,~\eqref{eq:permutation} is applied to the corresponding submatrix of $\vC$.},
\begin{align}
    \label{eq:permutation}
    \hat{\vP} \in \argmax_{\vP \in \mathcal{P}_m} \tr(\vC \vP),
\end{align}
where $\mathcal{P}_m$ denotes the set of all $m \times m$ permutation matrices. The trace sums the correlations of image and text coordinates matched by $\vP$. We compute this assignment using the Hungarian algorithm~\citep{kuhn1955hungarian}, as in prior work~\citep{fel2025archetypal,papadimitriou2025interpreting}. We then apply the permutation only to the SAE for the text modality as
\begin{align}
    \label{eq:permuted}
    \tilde{\vz}(\vy; \vW) := \sigma\big(\hat{\vP}^\top \vW^\top \vy\big),
    \qquad
    \tilde{\vy}(\vy; \vW) := \vW \hat{\vP} \tilde{\vz}(\vy;\vW).
\end{align}
This alignment simply permutes the learned features of the text modality without changing their directions. Consequently, it preserves the feature geometry of each modality while assigning matched features from the image and text modalities to the same coordinates in the latent codes. We focus on one-to-one matching, but alternative matching methods can also be used. Appendix~\ref{app:alignment_validation} shows that post-hoc alignment outperforms the shared-SAE baselines in cross-modal retrieval across all evaluated matching methods.

\paragraph{Inference.}
Embeddings are encoded into latent codes using modality-specific SAEs, and the permutation is applied to align their indices as in~\eqref{eq:permuted}. This alignment enables cross-modal tasks such as retrieval in the aligned latent space, while encoding and decoding remain modality-specific.

\section{Experiments}
\label{sec:experiments}

This section validates the theoretical findings in Section~\ref{sec:theory} and evaluates the approach proposed in Section~\ref{sec:ours}. Section~\ref{subsec:synthetic_exp} uses synthetic embeddings where ground-truth feature directions and cross-modal correspondences are known, allowing us to directly test the theoretical predictions. Section~\ref{subsec:real_exp} evaluates our method on real-world datasets to assess its effectiveness on cross-modal tasks.

\subsection{Validation on Synthetic Embeddings}
\label{subsec:synthetic_exp}

\input{figures/synthetic_specific}

\input{figures/synthetic_compare_baselines}

We validate two key design choices of the method proposed in Section~\ref{sec:ours}. First, we test whether training modality-specific SAEs better preserves image and text feature directions than training a shared SAE. Second, we test whether post-hoc alignment can align corresponding latent coordinates without sacrificing feature recovery. These design choices are motivated by the analysis in Section~\ref{sec:theory}.

\paragraph{Setup.}
To test the first choice, we compare a \emph{shared SAE}, which uses one weight matrix for both modalities, with \emph{modality-specific SAEs}, which use separate weight matrices for image and text embeddings. For a fair comparison, we match the total number of trainable parameters by using $m/2$ latent coordinates for each modality-specific SAE when the shared SAE uses $m$. Thus, any performance difference reflects how capacity is distributed across modalities, rather than model size.

To test the second choice, we compare our post-hoc alignment method with two auxiliary-loss baselines~\citep{dhimoila2026crossmodal, kaushik2026learning}. We report the average over three independent runs. Each latent coordinate is independently set to zero with probability $0.99$ and otherwise drawn from an exponential distribution with rate $1$. Details on synthetic embedding generation and training setup are provided in Appendix~\ref{app:synth_detail}.

\paragraph{Metrics.}
We evaluate each method along three axes: \emph{Reconstruction Error} and \emph{Feature Recovery Error} for reconstruction quality; \emph{Latent Alignment Error} and \emph{Feature Alignment Error} for cross-modal alignment; and \emph{Feature Collapse Rate} for whether same-concept features across modalities are collapsed into a single learned direction.
See Appendix~\ref{app:synth_metrics} for formal definitions.

\paragraph{Modality-specific sparse autoencoders better preserve feature directions under heterogeneity.}
Figure~\ref{fig:synthetic_specific} reports the results obtained by varying the cosine distance $d_{\cos}$ between corresponding cross-modal features $\vphi_i$ and $\vpsi_i$. In Figure~\ref{fig:synthetic_specific:a}, the shared SAE exhibits feature collapse when image and text feature directions are geometrically close but not identical. Modality-specific SAEs avoid this collapse because each modality has its own SAE. Figures~\ref{fig:synthetic_specific:b} and~\ref{fig:synthetic_specific:c} show that this collapse leads to larger reconstruction and feature recovery errors. This is consistent with the analysis in Section~\ref{sec:theory}, where representing different feature directions with a single decoder column increases reconstruction loss. At zero distance ($d_{\cos}=0$), the image and text feature directions coincide ($\forall i\in[n], \vphi_i=\vpsi_i$), so there is no cross-modal feature heterogeneity. In this idealized case, the shared SAE can represent the same concept with a single coordinate, which explains its lower reconstruction and recovery errors. Overall, these results show that modality-specific SAEs better preserve feature directions when semantically corresponding features differ across modalities, leading to better feature recovery.

\paragraph{Post-hoc alignment improves alignment without sacrificing feature recovery.}
Figure~\ref{fig:synthetic_compare_baselines} reports the results obtained by varying the auxiliary-loss weight $\lambda$. We fix the cosine distance between corresponding cross-modal features at $0.5$ to reflect the level of cross-modal feature heterogeneity observed in Figure~\ref{fig:multi_density}. For the Iso-Energy alignment loss~\citep{dhimoila2026crossmodal}, increasing $\lambda$ improves cross-modal alignment but degrades feature recovery, consistent with the reconstruction-alignment trade-off shown in Proposition~\ref{thm:prev:align}. For the group-sparse loss~\citep{kaushik2026learning}, increasing $\lambda$ eventually pushes the SAE into a degenerate regime where both feature recovery and alignment deteriorate, consistent with Proposition~\ref{thm:prev:group}. In contrast, our post-hoc approach preserves the feature recovery achieved by modality-specific SAEs before alignment, while achieving better alignment than the auxiliary-loss baselines. These results support decoupling cross-modal alignment from the reconstruction objective.

\subsection{Evaluation on Real-World Data}
\label{subsec:real_exp}

We evaluate our method on real-world embeddings extracted from a pre-trained VLM. We assess whether latent codes support cross-modal tasks. Following prior work~\citep{kaushik2026learning}, we evaluate whether estimated latent codes $\tilde{\vz}$ in~\eqref{eq:encoder} provide useful representations for cross-modal retrieval and concept-conditioned retrieval.

\paragraph{Experimental protocol.}
We use \texttt{CC-3M}~\citep{sharma2018conceptual} as a paired image-text dataset. From this dataset, we extract image and text embeddings using the \texttt{CLIP} model~\citep{radford21clip} with the ViT-B/32 architecture, and train SAEs on the resulting paired embeddings. We compare the method proposed in Section~\ref{sec:ours} with two auxiliary-loss baselines~\citep{dhimoila2026crossmodal, kaushik2026learning}. To ensure a fair comparison between modality-specific and shared SAE approaches, all methods use the same total number of SAE parameters. We follow prior training protocols~\citep{papadimitriou2025interpreting} and report averages over three runs. Details are provided in Appendix~\ref{app:real_detail}.

\input{tables/real_sae_full}

\paragraph{Our method improves cross-modal performance while preserving reconstruction.}
We evaluate reconstruction quality and performance across two cross-modal tasks. Reconstruction quality is measured by the mean squared error between input embeddings and their reconstructions. For cross-modal retrieval on \texttt{MS-COCO}, we rank image and text latent codes by cosine similarity and report Recall@$k$. For zero-shot image classification on \texttt{ImageNet1K}~\citep{deng2009imagenet}, we report top-1 accuracy.

Table~\ref{tab:real_sae_compact} reports the results for all methods. Our method achieves the lowest reconstruction error and the strongest cross-modal retrieval performance, while remaining competitive on zero-shot classification. These results show that post-hoc alignment improves cross-modal performance without sacrificing the reconstruction quality of modality-specific SAEs.

\input{figures/steering}
\vspace{-4pt}

\paragraph{Our method improves concept-conditioned retrieval.}
We evaluate whether aligned latent coordinates support concept-conditioned image retrieval by modifying query embeddings. For each of the $80$ object categories in \texttt{COCO}, we select the text latent coordinate with the largest difference in mean activation between captions mentioning the target category and randomly sampled captions that do not. We use the decoder column of the aligned image coordinate as the steering vector. We add this vector to the embeddings of $100$ test images without the target category and rank all test images by cosine similarity to each modified embedding. Successful retrieval places images containing the target category at higher ranks.

Figure~\ref{fig:steering} reports the retrieval performance and qualitative examples of concept-conditioned retrieval. Our method achieves the highest mean average precision and mean reciprocal rank over retrieved test images that contain the target concept. The examples further show that our method retrieves images containing the target concept more consistently. These results indicate that post-hoc latent coordinate alignment enables more effective concept-conditioned retrieval.

We further evaluate our method across VLMs and SAE configurations in Appendix~\ref{app:additional_real_results}, and assess latent monosemanticity in Appendix~\ref{app:monosemanticity_detail}.

\section{Conclusion}
We studied how features are organized across modalities in joint embedding spaces of vision--language models. We characterized \emph{cross-modal feature heterogeneity}, where the same semantic concept can have different feature directions across image and text modalities. We showed that this heterogeneity can give rise to modality split in SAEs, and that existing auxiliary-loss approaches trade reconstruction quality for latent alignment. Motivated by this limitation, we proposed a simple approach that trains modality-specific SAEs and aligns their latent codes through coactivation. In experiments, our method preserves reconstruction quality and improves cross-modal alignment. In summary, effective alignment in joint embedding spaces should preserve each modality's feature geometry.

%% file: figures/problem.tex
\begin{figure*}[t]
    \centering
    \begin{subfigure}[t]{0.5\textwidth}
        \centering
        \includegraphics[width=\textwidth]{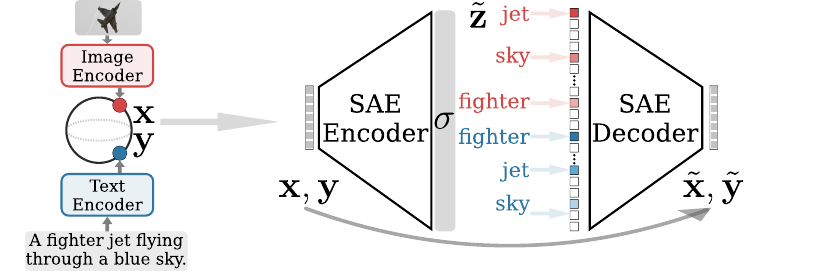}
        \caption{Modality split in sparse autoencoder}
        \label{fig:problem:a}
    \end{subfigure}
    \hfill
    \begin{subfigure}[t]{0.48\textwidth}
        \centering
        \includegraphics[width=\textwidth]{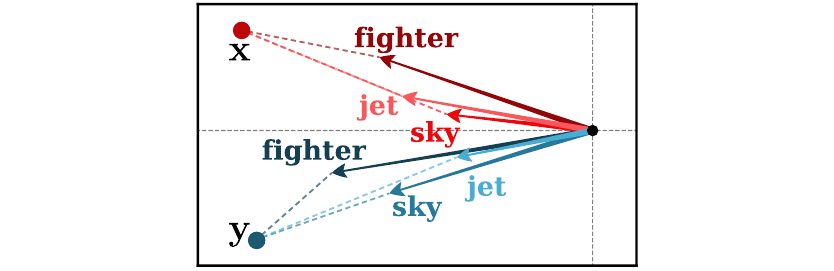}
        \caption{Cross-modal feature heterogeneity}
        \label{fig:problem:b}
    \end{subfigure}
    \caption{
    Modality split and cross-modal feature heterogeneity in joint embedding spaces.
    (a) An SAE encodes an image or text embedding ($\vx$ or $\vy$) from a VLM into a sparse latent code $\tilde{\vz}$ and then reconstructs the original input as $\tilde{\vx}$ or $\tilde{\vy}$. Ideally, the same concept should activate the same coordinate in $\tilde{\vz}$ across modalities. However, prior work observes the \emph{modality split}, where the same concept (\eg \emph{sky}, \emph{jet}, or \emph{fighter}) activates different coordinates for images (red) and text (blue). (b) We show that this split is driven by \emph{cross-modal feature heterogeneity}, where corresponding features (shown as arrows) fail to align directionally across modalities within the embedding space.
    }
    \label{fig:problem}
\end{figure*}

%% file: figures/multi_density.tex
\begin{figure*}[t]
  \centering
  \includegraphics[width=1.0\textwidth]{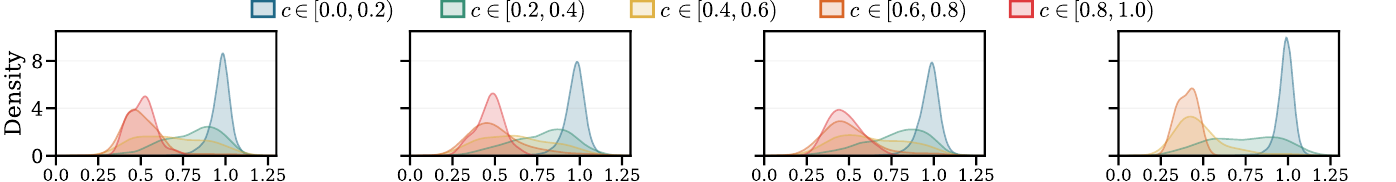}
  \resizebox{\textwidth}{!}{%
    \begin{tabular}{*{4}{p{0.25\textwidth}}}
        (a) \makecell{\texttt{CLIP}~\citep{radford21clip} ViT-B/32} &
        (b) \makecell{\texttt{MetaCLIP}~\citep{chuang2026meta} B/32} &
        (c) \makecell{\texttt{OpenCLIP}~\citep{cherti2023reproducible} B/32} &
        (d) \makecell{\texttt{SigLIP2}~\citep{tschannen2025siglip} Base}
    \end{tabular}%
}
  \caption{
    Distribution of cosine distances $d_{\cos}$ between image--text feature pairs estimated from embeddings of four VLMs. We group image--text feature pairs ($\hat{\vphi}_i,\hat{\vpsi}_j$) by their coactivation correlation $c_{i,j}$ in~\eqref{eq:correlation}, shown in different colors. The value $c_{i,j}$ measures how strongly the $i$-th image feature and the $j$-th text feature coactivate on paired image--text embeddings, so pairs with larger correlations are more likely to represent the same shared concept. Across all models, the distribution of distances $d_{\cos}$ remains centered around a positive value, even for high-correlation pairs ($c\geq0.8$, shown in red). This observation supports the presence of cross-modal feature heterogeneity.
  }
  \label{fig:multi_density}
\end{figure*}

%% file: figures/method.tex
\begin{figure*}[t]
  \centering
  \includegraphics[width=0.95\textwidth]{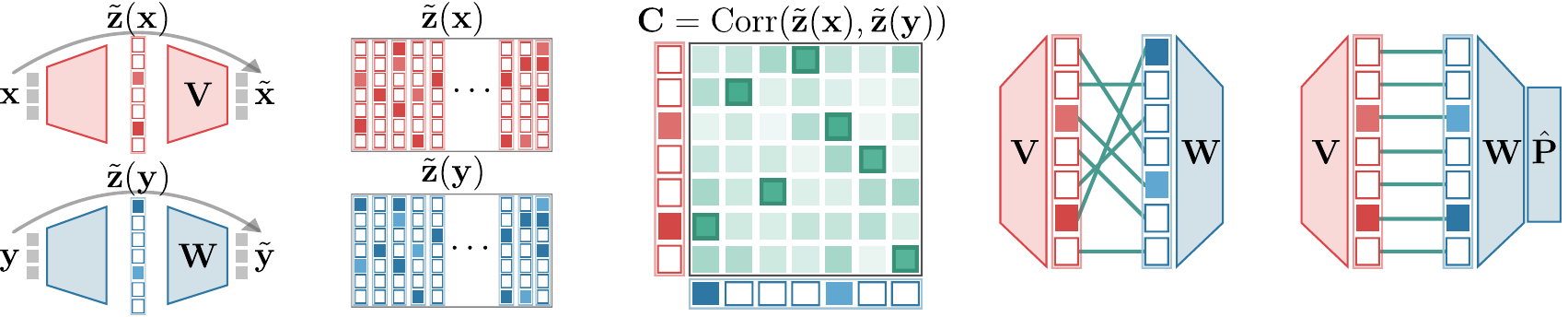}
  \resizebox{\textwidth}{!}{%
    \begin{tabular}{*{5}{p{0.2\textwidth}}}
        1. \makecell{Train SAEs} &
        2. \makecell{Estimate \\ Latent Codes} &
        3. \makecell{Compute \\ Correlations} &
        4. \makecell{Match \\ Coordinates} &
        5. \makecell{Permute SAE}
    \end{tabular}%
}
\caption{
    Overview of our approach.
    We first train modality-specific SAEs to reconstruct image and text embeddings, obtaining latent codes $\tilde{\vz}(\vx)$ and $\tilde{\vz}(\vy)$. We then compute their correlation matrix $\vC$ in~\eqref{eq:correlation} over paired training data and apply the Hungarian algorithm to obtain the assignment $\hat{\vP}$ in~\eqref{eq:permutation}. Finally, we reindex the text latent coordinates using $\hat{\vP}$, aligning corresponding image and text features to shared latent indices without altering the learned feature directions.
}
\label{fig:method}
\end{figure*}

%% file: figures/synthetic_specific.tex
\begin{figure*}[t]
  \centering
  \includegraphics[width=\textwidth]{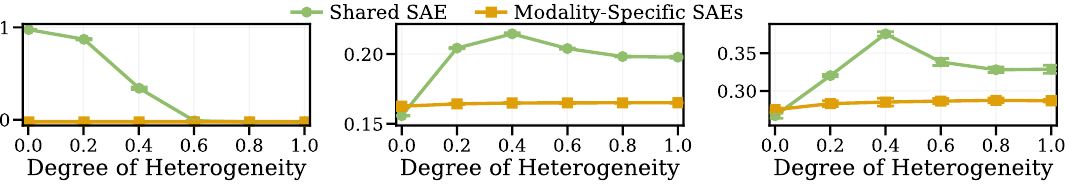}
\begin{subfigure}[b]{0.33\textwidth}
\centering
\caption{Feature collapse rate $(\downarrow)$}
\label{fig:synthetic_specific:a}
\end{subfigure}\hfill
\begin{subfigure}[b]{0.33\textwidth}
\centering
\caption{Reconstruction error $(\downarrow)$}
\label{fig:synthetic_specific:b}
\end{subfigure}\hfill
\begin{subfigure}[b]{0.33\textwidth}
\centering
\caption{Feature recovery error $(\downarrow)$}
\label{fig:synthetic_specific:c}
\end{subfigure}
  \caption{
  Comparison between a shared SAE trained on both modalities \textcolor{cShared}{(green)} and modality-specific SAEs trained separately for each modality \textcolor{cSeparated}{(yellow)} as the cross-modal feature distance $d_{\cos}$ varies. Each point corresponds to SAEs trained on synthetic embeddings sampled from ground-truth feature matrices $\vPhi$ and $\vPsi$, where the x-axis shows the cosine distance $d_{\cos}$ between corresponding image and text feature directions. When corresponding directions are distinct but geometrically close, the shared SAE tends to collapse them into a single learned direction, leading to higher reconstruction and feature recovery errors. In contrast, modality-specific SAEs avoid this collapse and better preserve feature directions under cross-modal feature heterogeneity, which corresponds to nonzero distance.
  }
  \label{fig:synthetic_specific}
\end{figure*}

%% file: figures/synthetic_compare_baselines.tex
\begin{figure*}[t]
  \centering
  \includegraphics[width=\textwidth]{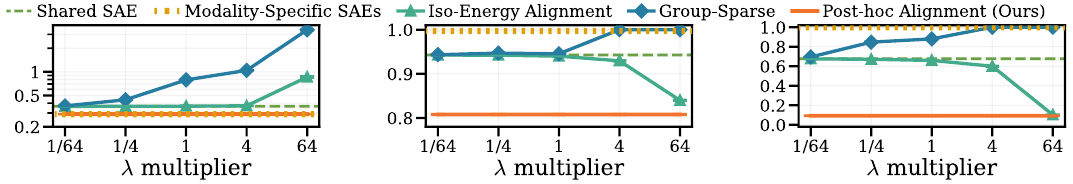}
\begin{subfigure}[b]{0.34\textwidth}
\centering
\caption{Feature recovery error $(\downarrow)$}
\label{fig:synthetic_compare_baselines:a}
\end{subfigure}\hfill
\begin{subfigure}[b]{0.33\textwidth}
\centering
\caption{Latent alignment error $(\downarrow)$}
\label{fig:synthetic_compare_baselines:b}
\end{subfigure}\hfill
\begin{subfigure}[b]{0.33\textwidth}
\centering
\caption{Feature alignment error $(\downarrow)$}
\label{fig:synthetic_compare_baselines:c}
\end{subfigure}
    \caption{
    Comparison between our post-hoc alignment method \textcolor{cPosthocAlign}{(orange)} and two auxiliary-loss baselines: Iso-Energy alignment \textcolor{cIsoEnergy}{(teal)} and group-sparse loss \textcolor{cGroupSparse}{(blue)}. The weight $\lambda$ of the auxiliary loss is shown as a multiplier of the default value in the original papers. For reference, we also include the shared SAE \textcolor{cShared}{(green)} and modality-specific SAEs \textcolor{cSeparated}{(yellow)}, neither of which depend on $\lambda$. Our method achieves both lower feature recovery error and lower alignment error than the baselines.
}
\label{fig:synthetic_compare_baselines}
\end{figure*}

%% file: tables/real_sae_full.tex
\begin{table}[t]
    \centering
    \caption{
    Reconstruction and alignment quality of various methods. Reconstruction quality is measured by mean squared error between the input embedding and its reconstruction. Cross-modal alignment is evaluated by Recall@$k$ for $k \in \{1, 5, 10\}$ on image-to-text and text-to-image retrieval tasks on \texttt{MS-COCO}~\citep{lin2014microsoft}, and by top-1 zero-shot classification accuracy on \texttt{ImageNet1K}~\citep{deng2009imagenet}. The best and second-best results for each metric are shown in bold and underlined, respectively.
    }
    {
    \setlength{\tabcolsep}{2pt}
    \resizebox{\linewidth}{!}{
    \begin{tabular}{l c ccc ccc c c}
        \toprule
        & \multicolumn{7}{c}{MS-COCO~\citep{lin2014microsoft}} & \multicolumn{2}{c}{ImageNet1K~\citep{deng2009imagenet}} \\
        \cmidrule(lr){2-8} \cmidrule(lr){9-10}
        & Recon. ($\downarrow$)
        & \multicolumn{3}{c}{Image-to-Text ($\uparrow$)}
        & \multicolumn{3}{c}{Text-to-Image ($\uparrow$)}
        & Recon. ($\downarrow$) & Zero-shot ($\uparrow$) \\
        \cmidrule(lr){2-2} \cmidrule(lr){3-5} \cmidrule(lr){6-8} \cmidrule(lr){9-9} \cmidrule(lr){10-10}
        Methods & MSE
        & Recall@1 & Recall@5 & Recall@10
        & Recall@1 & Recall@5 & Recall@10
        & MSE & Accuracy \\
        \midrule
        Shared SAE
        & 0.090
        & 6.1 {\tiny $(\pm 0.3)$} & 13.3 {\tiny $(\pm 0.8)$} & 17.8 {\tiny $(\pm 0.8)$}
        & 3.4 {\tiny $(\pm 0.2)$} & 9.5 {\tiny $(\pm 0.5)$} & 14.0 {\tiny $(\pm 0.8)$}
        & 0.118 & 15.7 {\tiny $(\pm 1.3)$} \\
        + Iso-Energy alignment loss
        & 0.091
        & 4.7 {\tiny $(\pm 1.9)$} & 10.6 {\tiny $(\pm 3.4)$} & 14.6 {\tiny $(\pm 4.4)$}
        & 2.7 {\tiny $(\pm 0.7)$} & 7.4 {\tiny $(\pm 2.1)$} & 10.8 {\tiny $(\pm 3.0)$}
        & 0.118 & 13.0 {\tiny $(\pm 2.3)$} \\
        + Group-sparse loss
        & 0.105
        & \underline{7.1} {\tiny $(\pm 0.1)$} & \underline{16.7} {\tiny $(\pm 0.4)$} & \underline{23.8} {\tiny $(\pm 0.3)$}
        & \underline{4.3} {\tiny $(\pm 0.1)$} & \underline{12.2} {\tiny $(\pm 0.1)$} & \underline{18.3} {\tiny $(\pm 0.2)$}
        & 0.134 & \textbf{26.6} {\tiny $(\pm 0.3)$} \\
        Modality-Specific SAEs
        & \textbf{0.089}
        & 0.0 {\tiny $(\pm 0.0)$} & 0.1 {\tiny $(\pm 0.1)$} & 0.2 {\tiny $(\pm 0.1)$}
        & 0.0 {\tiny $(\pm 0.0)$} & 0.1 {\tiny $(\pm 0.0)$} & 0.1 {\tiny $(\pm 0.1)$}
        & \textbf{0.116} & 0.1 {\tiny $(\pm 0.0)$} \\
        \rowcolor{gray!20}
        + Post-hoc Alignment (Ours)
        & \textbf{0.089}
        & \textbf{16.0} {\tiny $(\pm 0.5)$} & \textbf{34.1} {\tiny $(\pm 1.4)$} & \textbf{44.5} {\tiny $(\pm 1.6)$}
        & \textbf{11.4} {\tiny $(\pm 0.4)$} & \textbf{27.2} {\tiny $(\pm 1.2)$} & \textbf{37.0} {\tiny $(\pm 1.3)$}
        & \textbf{0.116} & \underline{25.1} {\tiny $(\pm 0.3)$} \\
        \bottomrule
    \end{tabular}
    }
    }
    \label{tab:real_sae_compact}
\end{table}

%% file: figures/steering.tex
\begin{figure}[t]
    \centering
    \includegraphics[width=0.9\linewidth]{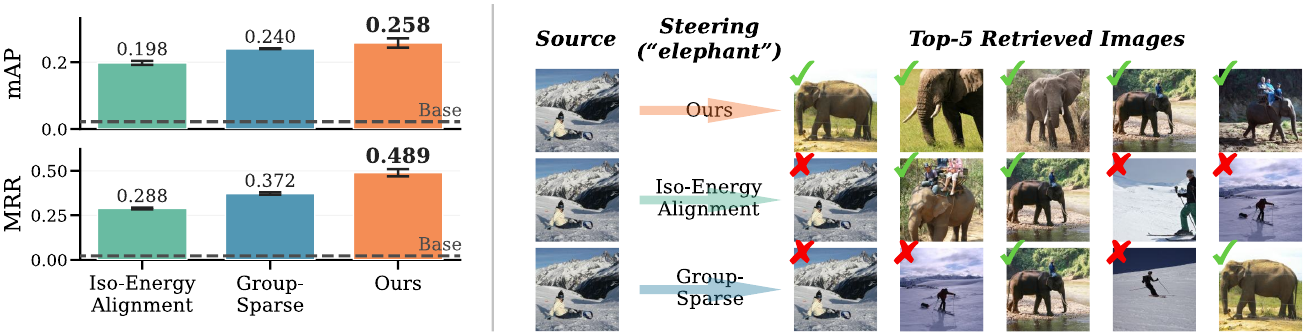}
    \caption{
    Image retrieval under concept latent steering on \texttt{MS-COCO}. We steer source image embeddings using target feature directions identified through aligned latent coordinates. Retrieval performance is measured by mean average precision (mAP) and mean reciprocal rank (MRR) over retrieved images containing the target concept. Qualitative examples show source images and their top retrieved images after steering. Our method steers images toward target concepts more effectively.}
    \label{fig:steering}
    \vspace{-4pt}
\end{figure}

%% file: texts/appendix.tex
\section{Limitations and Future Work}
\label{sec:limitation}

Our analysis relies on the linear representation hypothesis and simplified generative assumptions, which may not fully capture the complexity of real-world VLM embeddings. While our post-hoc alignment improves empirical performance, it assumes that feature correspondence can be inferred from coactivation statistics, which may become less accurate in noisy or weakly aligned settings.

Future work includes extending the analysis beyond the linear representation hypothesis and studying alignment under more complex forms of heterogeneity. While our current experiments focus on CLIP-like contrastive VLMs, another important direction is to evaluate whether the proposed method remains effective for representations extracted from broader VLM architectures, including LLM-based autoregressive VLMs such as LLaVA. Applying these ideas to larger-scale VLMs and downstream tasks may further clarify the role of feature-level alignment in multimodal systems. Moreover, post-hoc alignment could be used as an initialization for subsequent fine-tuning.

\section{Additional Discussion on Cross-Modal Feature Heterogeneity}
\label{app:heterogeneity}

\paragraph{A possible source of cross-modal feature heterogeneity.}
We next discuss why cross-modal feature directions may be unevenly aligned in VLMs trained with objectives such as contrastive learning. Such objectives encourage cross-modal similarity for paired samples, typically by increasing $\E[\vx^\top \vy]$. As the sparsity level $s \to 1$, we have
\begin{align}
    \E \big[ \vx^\top \vy \big]
    &=
    \tr\!\big( \vPhi^\top \vPsi \; \E[\vz\vz^\top] \big)
    =
    s^{n-1} (1-s) {\textstyle\sum_{i \in [n]}} \vphi_i^\top \vpsi_i \cdot \E\big[z_i^2\big| z_i \neq 0\big] + o(1-s)
    .
\end{align}
This expression suggests that each feature pair contributes to positive similarity in proportion to the second moment of its latent coordinate. Consequently, under finite data or finite model capacity, frequently activated features may receive stronger alignment pressure, whereas rarer features may remain less aligned. This provides one possible source of cross-modal feature heterogeneity.

\paragraph{Relation to modality gap.}
The following proposition shows that cross-modal feature heterogeneity is unavoidable whenever the embedding pair exhibits a nontrivial modality gap.

\begin{restatable}{proposition}{thmgap}
\label{thm:heterogeneity}
If there exists a modality gap between an embedding pair, i.e., $\E[\cos(\vx,\vy)] < 1$,
then there exists at least one latent concept that exhibits cross-modal feature heterogeneity.
\end{restatable}
\begin{proof}
For contradiction, assume that there is no cross-modal feature heterogeneity. Then $\vphi_i = \vpsi_i$ for all $i \in [n]$, which implies $\vPhi = \vPsi$. Hence, the embeddings satisfy $\vx = \vPhi \vz = \vPsi \vz = \vy$ almost surely. This gives $\cos(\vx,\vy)=1$ almost surely, contradicting the condition that $\E[\cos(\vx,\vy)] < 1$. Therefore, there must exist at least one $i \in [n]$ such that $\vphi_i \neq \vpsi_i$.
\end{proof}
This result formalizes that a modality gap at the embedding level must be reflected by at least one directional mismatch between corresponding image and text features under the linear representation hypothesis. The main question studied in this paper is how such heterogeneity affects multimodal SAEs.

\section{Proofs of Theoretical Results}
\label{sec:proofs}

We begin by clarifying notation. Let $\vzero_d$ and $\vzero_{m \times n}$ denote the zero vector in $\sR^d$ and the $m \times n$ zero matrix, respectively. For a matrix $\vZ$, $[\vZ]_{[:,i]}$ denotes its $i$-th column. Throughout the proofs, the Top-1 operator retains the largest positive activation and returns zero if all activations are nonpositive, and each nonzero decoder column has unit norm.

\thmlarge*
\begin{proof}
Although Theorem~\ref{thm:m-small} is stated for the regime $m<2n$, the same argument extends beyond this restriction. When $m \ge 2n$, it gives the present corollary. We give a direct proof to make the construction explicit.

For $\vz \in \sR^n_+$, define the event $\sS_k(\vz) := \{\#\{i \in [n] : z_i = 0\} = k\}$
for $k \in \{0\}\cup[n]$, which denotes that exactly $k$ entries of $\vz$ are zero.
From~\eqref{eq:sparsity}, $\Pr(\sS_k(\vz)) = \binom{n}{k}s^k(1-s)^{n-k}$.
In the sparsity regime $s \to 1$, the loss admits the expansion
\begin{align}
    \label{eq:proof:expansion}
    \loss_{\mathrm{rec}}(\vW;\, \vPhi)
    =
    s^{n-1}(1-s)
    \sum_{i \in [n]} \mu'_i\,\|\vphi_i - \vW\sigma(\vW^\top\vphi_i)\|_2^2
    + o(1-s),
\end{align}
where $\mu'_i := \E[z_i^2 \mid z_i \neq 0] > 0$ are finite constants for $i \in [n]$.
An expansion also holds for $\loss_{\mathrm{rec}}(\vW;\,\vPsi)$. Defining $\vtheta_i := \vphi_i$, $\vtheta_{n+i} := \vpsi_i$, and $\mu'_{n+i} := \mu'_i$ for $i \in [n]$, we obtain
\begin{align}
    \label{eq:proof:combined}
    \loss_{\mathrm{rec}}(\vW;\, \vPhi) + \loss_{\mathrm{rec}}(\vW;\, \vPsi)
    =
    s^{n-1}(1-s)
    \sum_{i \in [2n]} \mu'_i\,\|\vtheta_i - \vW\sigma(\vW^\top\vtheta_i)\|_2^2
    + o(1-s).
\end{align}
The same expansion applies even when either feature matrix contains zero columns, since these columns contribute no reconstruction error and only nonzero columns are assumed to have unit norm.

Since $s^{n-1}(1-s)>0$ and $\mu'_i>0$, the leading term in~\eqref{eq:proof:combined} is minimized to zero precisely when
\begin{align}
    \label{eq:proof:condition}
    \vtheta_i = \vW\sigma(\vW^\top\vtheta_i)
    \quad\text{for all } i \in [2n].
\end{align}
For any permutation matrix $\vP \in \sR^{m \times m}$, substituting $\hat{\vW} = [\vPhi \;\; \vPsi \;\; \vzero_{d \times (m-2n)}]\vP$ into~\eqref{eq:proof:condition} gives, for all $i \in [2n]$,
\begin{align}
    [\vPhi \;\; \vPsi \;\; \vzero_{d \times (m-2n)}]\vP\,
    \sigma\!\bigl(\vP^\top[\vPhi \;\; \vPsi \;\; \vzero_{d \times (m-2n)}]^\top\vtheta_i\bigr)
    =
    [\vPhi \;\; \vPsi]\,\sigma\!\bigl([\vPhi \;\; \vPsi]^\top\vtheta_i\bigr)
    =
    \vtheta_i.
\end{align}
Therefore, $\hat{\vW}$ minimizes the leading term in~\eqref{eq:proof:combined}, and $\loss_{\mathrm{rec}}(\hat{\vW};\, \vPhi) + \loss_{\mathrm{rec}}(\hat{\vW};\, \vPsi) = o(1-s)$.
\end{proof}

\thmsmall*
\begin{proof}
For each $i \in [n]$, let $\vtheta_i := \vphi_i$, $\vtheta_{n+i} := \vpsi_i$, and $\mu'_i := \E[z_i^2 \mid z_i \neq 0] > 0$, and let $\mu'_{n+i} := \mu'_i$. By the same expansion as in~\eqref{eq:proof:combined}, as $s \to 1$, we have
\begin{align}
    \label{eq:proof:small:expansion}
    \loss_{\mathrm{rec}}(\vW;\, \vPhi) + \loss_{\mathrm{rec}}(\vW;\, \vPsi)
    =
    s^{n-1}(1-s)
    \sum_{i \in [2n]} \mu'_i\,\|\vtheta_i - \vW\sigma(\vW^\top\vtheta_i)\|_2^2
    + o(1-s).
\end{align}
We first fix $[\vw_1 \cdots \vw_m]$. For each $i \in [2n]$, let $j^\star(i)\in\argmax_{j\in[m]}\vw_j^\top\vtheta_i$ and $\hat{\vw}_i:=\vw_{j^\star(i)}$. Then
\begin{align}
    \sum_{i \in [2n]}\mu'_i \|\vtheta_i - \vW\sigma(\vW^\top\vtheta_i)\|_2^2
    =
    \sum_{i \in [2n]}\mu'_i \|\vtheta_i - \hat{\vw}_i[\hat{\vw}_i^\top\vtheta_i]_{+}\|_2^2
    =
    \sum_{i \in [2n]}\mu'_i  - \sum_{i \in [2n]}\mu'_i [\vtheta_i^\top\hat{\vw}_i]_{+}^2,
\end{align}
where $[x]_+ := \max\{x,0\}$. Since $\sum_{i \in [2n]}\mu'_i$ is constant with respect to $\vW$, minimizing the leading term in~\eqref{eq:proof:small:expansion} is equivalent to maximizing
\begin{align}
    \mathcal{J}(\vW)
    :=
    \sum_{i \in [2n]}\mu'_i[\vtheta_i^\top\hat{\vw}_i]_{+}^2
    \le
    \sum_{i \in [2n]}\hat{\vw}_i^\top\vM_i\,\hat{\vw}_i,
\end{align}
where, for each $i \in [n]$, we define $\vM_i := \mu'_i \vphi_i\vphi_i^\top \in \sR^{d \times d}$ and $\vM_{n+i} := \mu'_i\vpsi_i\vpsi_i^\top \in \sR^{d \times d}$ for $i \in [n]$.
For each $j \in [m]$, define $\sA_j := \{i \in [2n] | j^\star(i)=j\}$, where $\sA_j$ contains the feature indices whose corresponding feature vectors are extracted by the column $\vw_j$. Then $(\sA_1,\cdots,\sA_m)$ forms a partition of $[2n]$, and
\begin{align}
    \mathcal{J}(\vW)
    \le
    \sum_{j \in [m]}
    \vw_j^\top \left(\sum_{i \in \sA_j}\vM_i\right)\vw_j.
\end{align}

For a fixed partition $(\sA_1,\ldots,\sA_m)$, maximizing the right-hand side decouples over the columns of $\vW$. Hence, for each $j \in [m]$,
\begin{align}
    \max_{\|\vw_j\|_2^2 = 1}
    \vw_j^\top \bigg(\sum_{i \in \sA_j}\vM_i\bigg)\vw_j
    =
    \lambda_{\max} \bigg(\sum_{i \in \sA_j}\vM_i\bigg),
\end{align}
where $\lambda_{\max}$ denotes the largest eigenvalue. The maximum is attained at a unit eigenvector corresponding to this largest eigenvalue.
Therefore, $\mathcal J(\vW)$ cannot exceed the maximum of $\sum_{j \in [m]} \lambda_{\max} \bigl(\sum_{i \in \sA_j}\vM_i\bigr)$ over all partitions $(\sA_1,\ldots,\sA_m)$ of $[2n]$. By the assumption in Theorem~\ref{thm:m-small}, $\mathcal J(\hat{\vW})$ attains this value, so $\hat{\vW}$ minimizes the leading term in~\eqref{eq:proof:small:expansion}.

Substituting back into~\eqref{eq:proof:small:expansion}, together with $\sum_{i \in [2n]}\mu'_i = 2\sum_{i \in [n]}\mu'_i$, gives
\begin{align}
    \loss_{\mathrm{rec}}(\hat{\vW};\, \vPhi) + \loss_{\mathrm{rec}}(\hat{\vW};\, \vPsi)
    =
    s^{n-1}(1-s)
    \bigg(
        2\sum_{i \in [n]}\mu'_i
        -
        \sum_{j \in [m]}\lambda_{\max}\!\bigg(\sum_{i \in \sA_j}\vM_i\bigg)
    \bigg)
    + o(1-s).
\end{align}
It remains to note that the leading coefficient is nonnegative. Since each $\vM_i$ is positive semidefinite,
\begin{align}
    \lambda_{\max}\!\bigg(\sum_{i \in \sA_j}\vM_i\bigg)
    \le
    \tr\!\bigg(\sum_{i \in \sA_j}\vM_i\bigg)
\end{align}
for each $j \in [m]$. Therefore, using the fact that $(\sA_1,\ldots,\sA_m)$ is a partition of $[2n]$ and that
$\|\vphi_i\|_2=\|\vpsi_i\|_2=1$, we obtain
\begin{align}
    \sum_{j \in [m]}
    \lambda_{\max}\!\bigg(\sum_{i \in \sA_j}\vM_i\bigg)
    \leq
    \sum_{j \in [m]}
    \operatorname{tr}\!\bigg(\sum_{i \in \sA_j}\vM_i\bigg)
    =
    \sum_{i \in [2n]}\operatorname{tr}(\vM_i)
    =
    2\sum_{i \in [n]}\mu'_i.
\end{align}

Moreover, if no two feature directions are collinear, \ie $\vtheta_i \neq \pm \vtheta_j$ for all distinct $i,j\in [2n]$, then the above inequality is strict. Since $m<2n$, at least one set $\sA_j$ contains two distinct indices. For this set, the matrix $\sum_{i\in\sA_j}\vM_i$ has rank at least two, and hence its largest eigenvalue is strictly smaller than its trace. Therefore,
\begin{align}
    \sum_{j \in [m]}
    \lambda_{\max}\!\bigg(\sum_{i \in \sA_j}\vM_i\bigg)
    <
    2\sum_{i \in [n]}\mu'_i.
\end{align}
Thus the leading coefficient of the loss is strictly positive. Consequently,
\begin{align}
    \loss_{\mathrm{rec}}(\hat{\vW};\, \vPhi)
    +
    \loss_{\mathrm{rec}}(\hat{\vW};\, \vPsi)
    > 0
\end{align}
for all sufficiently large $s<1$.

The same conclusion holds with zero columns if the remaining feature directions are pairwise noncollinear and outnumber the decoder columns, under the same eigenvector condition.
\end{proof}

\thmgroup*
\begin{proof}
Since $z \geq 0$, we have $\sigma(\vW^\top\vphi z) = z\,\sigma(\vW^\top\vphi)$ and $\sigma(\vW^\top\vpsi z) = z\,\sigma(\vW^\top\vpsi)$. Hence
\begin{align}
    &\loss_{\mathrm{rec}}(\vW;\vphi) + \loss_{\mathrm{rec}}(\vW;\vpsi)
    =
    \E[z^2]\Big(
        \|\vphi - \vW\sigma(\vW^\top\vphi)\|^2
        +
        \|\vpsi - \vW\sigma(\vW^\top\vpsi)\|^2
    \Big),
    \\
    &\E_z\!\Bigg[\sum_{j \in [m]}\sqrt{[\sigma(\vW^\top\vphi z)]_j^2
    + [\sigma(\vW^\top\vpsi z)]_j^2}\Bigg]
    =
    \E[z]\sum_{j \in [m]}
    \sqrt{[\sigma(\vW^\top\vphi)]_j^2 + [\sigma(\vW^\top\vpsi)]_j^2}.
\end{align}
Setting $\vu := \sigma(\vW^\top\vphi)$, $\vv := \sigma(\vW^\top\vpsi)$, and $\tilde\lambda := \lambda\,\E[z]/\E[z^2]$, dividing through by $\E[z^2]$ reduces the problem to minimizing
\begin{align}
    \mathcal{J}(\vW)
    :=
    \|\vphi - \vW\vu\|^2 + \|\vpsi - \vW\vv\|^2
    + \tilde\lambda\sum_{j \in [m]}\sqrt{[\vu]_j^2 + [\vv]_j^2}.
\end{align}

If $\vphi$ and $\vpsi$ activate different columns $i \neq j$, setting $[\vW]_{[:,i]} = \vphi$ and $[\vW]_{[:,j]} = \vpsi$ gives zero reconstruction error and group penalty $2$. Hence,
\begin{align}
    L_{\mathrm{sep}} := \mathcal{J}(\vW) = 2\tilde\lambda.
\end{align}

If instead both activate the same unit column $\vw$, writing $a := \vw^\top\vphi \geq 0$ and $b := \vw^\top\vpsi \geq 0$,
\begin{align}
    \mathcal{J}([\vw]) = 2 - (a^2 + b^2) + \tilde\lambda\sqrt{a^2 + b^2}.
\end{align}
With $\vM := \vphi\vphi^\top + \vpsi\vpsi^\top$, we have $a^2 + b^2 = \vw^\top\vM\vw$. Then the top eigenvector of $\vM$ is
\begin{align}
    \hat{\vw} := \frac{\vphi+\vpsi}{\|\vphi+\vpsi\|},
    \qquad
    a = b = \sqrt{\tfrac{1+\rho}{2}},
\end{align}
with eigenvalue $1+\rho$. By the Rayleigh quotient, the feasible range of $a^2 + b^2$ over unit $\vw$ satisfies $a^2 + b^2 \leq 1+\rho$, with equality at $\hat{\vw}$. Since $\mathcal{J}([\vw])$ depends on $\vw$ through $r := \sqrt{a^2+b^2}$, define
\begin{align}
    f(r) := 2 - r^2 + \tilde\lambda r,
    \qquad r \in [0,\sqrt{1+\rho}].
\end{align}
Since $f$ is concave in $r$, its minimum over $[0,\sqrt{1+\rho}]$ is attained at an endpoint.
Evaluating at $r=0$ and $r=\sqrt{1+\rho}$ gives
\begin{align}
    f(0) = 2,
    \qquad
    f(\sqrt{1+\rho}) = (1-\rho) + \tilde\lambda\sqrt{1+\rho}.
\end{align}
Thus
\begin{align}
    L_{\mathrm{sh}} := \min_{\vw}\mathcal{J}([\vw])
    = \min\big\{2,\; (1-\rho) + \tilde\lambda\sqrt{1+\rho}\big\}.
\end{align}

Finally, consider the inactive solution where
$\sigma(\vW^\top\vphi) = \vzero$ and $\sigma(\vW^\top\vpsi) = \vzero$. In this case,
\begin{align}
    L_{\mathrm{dead}} := \mathcal{J}(\vW) = \|\vphi\|^2 + \|\vpsi\|^2 = 2.
\end{align}

We now compare the three values:
\begin{align}
L_{\mathrm{sep}} = 2\tilde\lambda,
\quad
L_{\mathrm{sh}} = \min\big\{2,\; (1-\rho)+\tilde\lambda\sqrt{1+\rho}\big\},
\quad
L_{\mathrm{dead}} = 2.
\end{align}

If $\tilde\lambda < \tfrac{1-\rho}{2-\sqrt{1+\rho}}$, then $2\tilde\lambda < (1-\rho) + \tilde\lambda\sqrt{1+\rho}$ and $2\tilde\lambda < 2$;
hence $L_{\mathrm{sep}} < L_{\mathrm{sh}}$ and $L_{\mathrm{sep}} < L_{\mathrm{dead}}$.

If $\tfrac{1-\rho}{2-\sqrt{1+\rho}} < \tilde\lambda < \sqrt{1+\rho}$, then $(1-\rho) + \tilde\lambda\sqrt{1+\rho} < 2\tilde\lambda$ and $(1-\rho) + \tilde\lambda\sqrt{1+\rho} < 2$;
hence $L_{\mathrm{sh}} < L_{\mathrm{sep}}$ and $L_{\mathrm{sh}} < L_{\mathrm{dead}}$.

If $\tilde\lambda > \sqrt{1+\rho}$, then $2 < 2\tilde\lambda$ and $2 < (1-\rho) + \tilde\lambda\sqrt{1+\rho}$;
hence $L_{\mathrm{dead}} < L_{\mathrm{sep}}$ and $L_{\mathrm{dead}} = L_{\mathrm{sh}}$.

Substituting $\tilde\lambda = \lambda\,\E[z]/\E[z^2]$ yields the result.
When the shared solution is optimal, the reconstruction errors are
\begin{align}
    \loss_{\mathrm{rec}}(\hat{\vW};\vphi)
    =
    \E[z^2](1-a^2)
    =
    \tfrac{1-\rho}{2}\,\E[z^2],
    \qquad
    \loss_{\mathrm{rec}}(\hat{\vW};\vpsi)
    =
    \E[z^2](1-b^2)
    =
    \tfrac{1-\rho}{2}\,\E[z^2],
\end{align}
and when the inactive solution is optimal,
\begin{align}
    \loss_{\mathrm{rec}}(\hat{\vW};\vphi)
    =
    \loss_{\mathrm{rec}}(\hat{\vW};\vpsi)
    =
    \E[z^2].
\end{align}
\end{proof}

\thmalign*
\begin{proof}
Since $z \geq 0$, we have $\sigma(\vW^\top\vphi z) = z\,\sigma(\vW^\top\vphi)$ and $\sigma(\vW^\top\vpsi z) = z\,\sigma(\vW^\top\vpsi)$. Hence
\begin{align}
    \loss_{\mathrm{rec}}(\vW;\vphi) + \loss_{\mathrm{rec}}(\vW;\vpsi)
    &=
    \E[z^2]\Big(
        \|\vphi - \vW\sigma(\vW^\top\vphi)\|^2
        +
        \|\vpsi - \vW\sigma(\vW^\top\vpsi)\|^2
    \Big),
    \\
    \E_z\!\Big[\sigma(\vW^\top\vphi z)^\top\sigma(\vW^\top\vpsi z)\Big]
    &=
    \E[z^2]\,\sigma(\vW^\top\vphi)^\top\sigma(\vW^\top\vpsi).
\end{align}
Setting $\vu := \sigma(\vW^\top\vphi)$ and $\vv := \sigma(\vW^\top\vpsi)$, the loss factors as $\E[z^2]\,\mathcal{J}(\vW)$, where
\begin{align}
    \mathcal{J}(\vW) := \|\vphi - \vW\vu\|^2 + \|\vpsi - \vW\vv\|^2 - \lambda\,\vu^\top\vv.
\end{align}

If $\vphi$ and $\vpsi$ activate different columns $j_a \neq j_b$, setting $[\vW]_{[:,j_a]} = \vphi$ and $[\vW]_{[:,j_b]} = \vpsi$ gives zero reconstruction error and $\vu^\top\vv = 0$; hence
\begin{align}
    L_{\mathrm{sep}} := \mathcal{J}(\vW) = 0.
\end{align}

If instead both activate the same unit column $\vw$, writing $a := \vw^\top\vphi \geq 0$ and $b := \vw^\top\vpsi \geq 0$,
\begin{align}
    \mathcal{J}([\vw])
    = 2 - a^2 - b^2 - \lambda\,ab
    = 2 - \vw^\top\vM_\lambda\vw,
\end{align}
where $\vM_\lambda := \vphi\vphi^\top + \vpsi\vpsi^\top + \tfrac{\lambda}{2}(\vphi\vpsi^\top + \vpsi\vphi^\top)$.
Under $\rho \in (0,1)$ and $\lambda \geq 0$, the top eigenvector is
\begin{align}
    \hat{\vw} := \frac{\vphi+\vpsi}{\|\vphi+\vpsi\|},
    \qquad
    a = b = \sqrt{\tfrac{1+\rho}{2}}.
\end{align}
By the Rayleigh quotient, $\vw^\top\vM_\lambda\vw \leq (1+\rho)(1+\tfrac{\lambda}{2})$ for unit $\vw$, with equality at $\hat{\vw}$. Therefore
\begin{align}
    L_{\mathrm{sh}} := \min_{\vw}\mathcal{J}([\vw])
    = 2 - (1+\rho)(1 + \tfrac{\lambda}{2})
    = (1-\rho) - \lambda\,\tfrac{1+\rho}{2}.
\end{align}

Finally, consider the inactive solution where
$\sigma(\vW^\top\vphi) = \vzero$ and $\sigma(\vW^\top\vpsi) = \vzero$. In this case,
\begin{align}
    L_{\mathrm{dead}} := \mathcal{J}(\vW) = \|\vphi\|^2 + \|\vpsi\|^2 = 2.
\end{align}

We now compare the three values:
\begin{align}
L_{\mathrm{sep}} = 0,
\quad
L_{\mathrm{sh}} = (1-\rho) - \lambda\,\tfrac{1+\rho}{2},
\quad
L_{\mathrm{dead}} = 2.
\end{align}
Since $\lambda \geq 0$, we have $L_{\mathrm{dead}} = 2 > 0 = L_{\mathrm{sep}}$, so the dead solution
is never optimal. Comparing the remaining two, $L_{\mathrm{sh}} < L_{\mathrm{sep}}$ if and only if
$\lambda > \tfrac{2(1-\rho)}{1+\rho} = \lambda^\star(\rho)$.

When $\lambda > \lambda^\star(\rho)$, the reconstruction errors at the shared minimizer are
\begin{align}
    \loss_{\mathrm{rec}}(\hat{\vW};\vphi)
    =
    \E[z^2](1-a^2)
    =
    \tfrac{1-\rho}{2}\,\E[z^2],
    \qquad
    \loss_{\mathrm{rec}}(\hat{\vW};\vpsi)
    =
    \E[z^2](1-b^2)
    =
    \tfrac{1-\rho}{2}\,\E[z^2].
\end{align}
\end{proof}

\clearpage
\section{Experimental Details}
\label{sec:exp_details}

This section provides details for the empirical analyses in Sections~\ref{sec:heterogeneity:definition} and~\ref{sec:experiments}. The code is available at \url{https://github.com/JiH00nKw0n/cross_modal_feature_heterogeneity}. All experiments are conducted on a single NVIDIA A100 GPU.

\subsection{Experimental Details for Measuring Cross-Modal Feature Heterogeneity}
\label{sec:exp:detail:sec3}

This section provides experimental details for Figure~\ref{fig:multi_density} in Section~\ref{sec:heterogeneity:definition}, which reports the distribution of cosine distances between estimated image and text feature directions grouped by coactivation correlation.
Recall that we use decoder columns $\hat{\vphi}_i$ and $\hat{\vpsi}_j$ of modality-specific SAEs as estimates of image and text feature directions, and that the coactivation correlation $c_{i,j} := [\vC]_{i,j}$ between the $i$-th image latent and the $j$-th text latent on paired embeddings~\eqref{eq:correlation} serves as a proxy for semantic correspondence.
For each pair $(i,j)$, we measure the cosine distance between $\hat{\vphi}_i$ and $\hat{\vpsi}_j$.

We group all $m^2$ index pairs $(i,j)$\footnote{\label{fnm:1}We exclude dead latent coordinates, which are zero across all inputs, so fewer than $m^2$ pairs remain in practice.} by their correlation value $c_{i,j}$ and examine the distribution of cosine distances between the corresponding feature vectors.
A larger $c_{i,j}$ indicates that the corresponding image and text features are more likely to represent the same shared concept. Under perfect cross-modal feature alignment, their cosine distances would concentrate near zero.
We train modality-specific SAEs with latent dimension $m = 8192$, using the Top-$K$~\citep{makhzani2013k} activation function with $K=8$.
We follow the training protocol of~\citet{papadimitriou2025interpreting}.
Specifically, we use the AdamW optimizer with learning rate $5{\times}10^{-4}$ and weight decay $10^{-5}$, batch size $1024$, and a cosine schedule with $5\%$ linear warmup over $30$ training epochs.

\paragraph{Validation using object annotations.}
\label{app:annotation_detail}
The SAEs and the post-hoc permutation are trained on \texttt{CC-3M}. We evaluate semantic correspondence using the \texttt{MS-COCO} categories selected in Section~\ref{sec:heterogeneity:definition}. For category $c$ and latent coordinate $j$, define
\begin{equation}
    \operatorname{AUC}_{j}(c)
    = \Pr\!\left(\tilde z_j(s^+) > \tilde z_j(s^-)\right)
    + \frac{1}{2}\Pr\!\left(\tilde z_j(s^+) = \tilde z_j(s^-)\right),
\end{equation}

where $s^+$ and $s^-$ are samples associated and not associated with category $c$, respectively. For images, association means that an annotated object of category $c$ is present. For captions, it means that the caption mentions the category. We independently select the coordinate maximizing this AUC in each modality. To evaluate the learned correspondence, a category is counted as an exact match only when the permutation pairs its independently selected image and text coordinates. For a shared SAE, the selected coordinates must have the same index.

\subsection{Implementation Details for Synthetic Data Experiments}
\label{app:synth_detail}

This section specifies the data-generation process, SAE architectures, and optimization hyperparameters for the synthetic data experiments shown in Section~\ref{subsec:synthetic_exp}.
We repeat each experiment over three runs, reporting the average.

\paragraph{Dataset.}
We generate each pair of synthetic embeddings, following prior work~\citep{goodfellow2012large, sheikh2014truncated}.
We first generate the sparse latent code $\vz \in \mathbb{R}_+^n$ by independently setting each coordinate to zero with probability $s = 0.99$ and otherwise sampling it from an exponential distribution with rate $\beta = 1$.
The embedding dimension is $d = 256$, and the feature matrices $\vPhi, \vPsi \in \mathbb{R}^{d \times n}$ each contain $n = n_S + n_I + n_T = 2048$ columns in total, where $n_S = 1024$ is the number of shared concepts present in both modalities, $n_I = 512$ is the number of image-only concepts (with $\vpsi_i = \mathbf{0}$), and $n_T = 512$ is the number of text-only concepts (with $\vphi_i = \mathbf{0}$).
For each shared concept $i \in [n_S]$, we sample $\vphi_i$ and $\vpsi_i$ as unit-norm vectors satisfying $\cos(\vphi_i, \vpsi_i) = \alpha$ for a prescribed cross-modal feature alignment $\alpha$.

All nonzero feature directions are unit-norm.
Apart from the prescribed alignment $\cos(\vphi_i, \vpsi_i) = \alpha$ between matched shared concepts, every pair of directions belonging to distinct concepts $i \neq j$ has bounded pairwise interference, $|\cos(\vphi_i, \vphi_j)|,\, |\cos(\vpsi_i, \vpsi_j)|,\, |\cos(\vphi_i, \vpsi_j)| \le 0.30$.
The paired embeddings are then formed as $\vx = \vPhi \vz + \boldsymbol\epsilon_I$ and $\vy = \vPsi \vz + \boldsymbol\epsilon_T$ via~\eqref{eq:embeddings}, with independent observation noise $\boldsymbol\epsilon_I, \boldsymbol\epsilon_T \sim \mathcal{N}(\boldsymbol{0}, \sigma_{\mathrm{obs}}^2 \mathbf{I})$ of standard deviation $\sigma_{\mathrm{obs}} = 0.05$.
We generate $50{,}000$ paired embeddings for training and $10{,}000$ for evaluation.

\paragraph{Models.}
The shared SAE, which uses a single SAE across both modalities (shown in Figures~\ref{fig:synthetic_specific} and~\ref{fig:synthetic_compare_baselines}), has latent dimension $m = 8192$ and uses the Top-$K$ activation function~\citep{makhzani2013k} with $K = 16$.
Each modality-specific SAE, which uses a separate SAE for each modality, uses latent dimension $m = 4096$, so that the total number of learnable parameters matches that of the shared SAE.
The two baselines (Iso-Energy Alignment~\citep{dhimoila2026crossmodal} and Group-Sparse~\citep{kaushik2026learning}), shown in Figure~\ref{fig:synthetic_compare_baselines}, augment the shared SAE by training with an auxiliary loss.

\paragraph{Training.}
All SAEs are trained with AdamW ($\beta_1 = 0.9$, $\beta_2 = 0.999$) at a constant learning rate of $5 \times 10^{-4}$ with no warmup.
We use a batch size of $256$ and train for $10$ epochs.

\subsection{Evaluation Metrics for Synthetic Data Experiments}
\label{app:synth_metrics}

We provide formal definitions of the five metrics introduced in Section~\ref{subsec:synthetic_exp}.

We assess each method using these metrics: (i) and (ii) measure how well the SAE reconstructs embeddings and recovers features, (iii) and (iv) measure how well it aligns latent codes across modalities, and (v) measures whether cross-modal features collapse into a single feature.
\begin{enumerate}[label=(\roman*)]
    \item \emph{Reconstruction Error.} We measure the mean squared error (MSE) between input embeddings $\vx$ (or $\vy$) and their reconstructions $\tilde{\vx}$ (or $\tilde{\vy}$) on an evaluation set, averaged over both modalities.
    \begin{equation}
        \frac{1}{2}\;
        \E
        \bigg[
        \norm{\vx-\tilde\vx}_2^2 + \norm{\vy-\tilde\vy}_2^2
        \bigg]
        =
        \frac{1}{2}\;
        \E
        \bigg[
        \norm{\vx- \vW\sigma\big(\vW^\top \vx\big)}_2^2 + \norm{\vy-\vW\sigma\big(\vW^\top \vy\big)}_2^2
        \bigg]
        .
    \end{equation}

    \item \emph{Feature Recovery Error.} Similar to Reconstruction Error, but using the feature direction vectors $\vphi_i$ (or $\vpsi_i$) as inputs, we measure their reconstruction MSE averaged over both modalities.
\begin{equation}
    \frac{1}{2n_{\mathrm S}}\sum_{i\in[n_{\mathrm S}]}
    \bigg(
    \norm{\vpsi_i- \vW\sigma\big(\vW^\top \vpsi_i\big)}_2^2 + \norm{\vphi_i-\vW\sigma\big(\vW^\top \vphi_i\big)}_2^2
    \bigg)
    .
\end{equation}

    \item \emph{Latent Alignment Error.} We measure the cosine distance between the image and text latent codes $\tilde{\vz}_{\mathrm I}(\vx)$ and $\tilde{\vz}_{\mathrm T}(\vy)$ obtained from paired embeddings $(\vx, \vy)$ on an evaluation set.
\begin{equation}
    1-
    \E
    \left[
    \cos\big(\tilde{\vz}(\vx),\tilde{\vz}(\vy)\big)
    \right]
    =
    1-
    \E
    \left[
    \cos\big(\sigma(\vW^{\!\top}\vx),\sigma(\vW^{\!\top}\vy)\big)
    \right].
\end{equation}

    \item \emph{Feature Alignment Error.} Similar to Latent Alignment Error, but using features of the same concept $(\vphi_i,\vpsi_i)$ as inputs, we measure the cosine distance between the latent codes $\tilde{\vz}_{\mathrm I}(\vphi_i)$ and $\tilde{\vz}_{\mathrm T}(\vpsi_i)$.
\begin{equation}
    1-
    \frac{1}{n_{\mathrm S}}\sum_{i\in[n_{\mathrm S}]}
    \cos\!\big(\sigma(\vW^{\!\top}\vphi_i),\;\sigma(\vW^{\!\top}\vpsi_i)\big).
\end{equation}

    \item \emph{(Cross-Modal) Feature Collapse Rate.} We measure the fraction of feature pairs $(\vphi_i,\vpsi_i)$ that represent the same concept across modalities and are assigned to the same latent coordinate, indicating whether the SAE merges two modality-specific features into a single learned feature.
\begin{equation}
    \frac{1}{n_{\mathrm S}}\sum_{i\in[n_{\mathrm S}]}
    \mathbf{1}\Big[
    \argmax_{j\in[m]} \cos\!\big(\vw_{j}, \vphi_i \big)
     =
     \argmax_{j\in[m]} \cos\!\big(\vw_{j}, \vpsi_i \big)
    \Big]
    \;\in\;[0,1].
\end{equation}

\end{enumerate}
Note that metrics (ii), (iv), and (v) require access to the ground-truth features $(\vPhi,\vPsi)$ and thus apply only in the synthetic setting.

\subsection{Implementation Details for Real-World Experiments}
\label{app:real_detail}

\paragraph{Implementation Details.}
We use 2.82M image-caption pairs from the \texttt{CC-3M}~\citep{sharma2018conceptual} dataset after removing invalid URLs. We use the AdamW optimizer~($\beta_1 = 0.9$, $\beta_2 = 0.999$) with a learning rate of $5\times 10^{-4}$, weight decay of $10^{-5}$, a batch size of $1024$, gradient clipping at a norm of $1.0$, and a cosine schedule with $5\%$ linear warmup followed by cosine decay to zero.
We use latent dimension $m=8192$ and Top-$K$ activation~\citep{makhzani2013k} with $K=32$, and we adopt the training protocol of~\citet{papadimitriou2025interpreting}. 

\paragraph{Licenses of datasets.}
The \texttt{CC-3M} dataset is released by Google for free use under the terms of its repository license. The \texttt{MS-COCO} dataset is released under \texttt{CC BY 4.0}.

\section{Additional Experimental Results}
\label{app:extended_experiments}

\subsection{Semantic Validation Using Object Annotations}
\label{app:semantic_validation}

We test whether post-hoc alignment pairs image and text features independently identified as representing the same annotated category, using the protocol in Appendix~\ref{app:annotation_detail}. Among the $65$ categories, modality-specific SAEs with post-hoc alignment match $29$, compared with $16$ for a shared SAE and $18$ and $5$ for shared SAEs trained with iso-energy alignment loss and group-sparse loss, respectively.

\subsection{Training Variability and Feature Stability}
\label{app:seed_stability}

\begin{table}[h]
\centering
\captionsetup{font=small}
\caption{Mean cosine distances for coactivating feature pairs.}
\label{tab:directional_validation}
\captionsetup[subtable]{position=top,font=small,labelfont=normalfont,justification=centering,singlelinecheck=false}
\begin{subtable}[t]{0.47\linewidth}
\centering
\caption{Within and across modalities}
\label{tab:seed_control_main}

\small
\setlength{\tabcolsep}{3pt}
\begin{tabular}{@{}llc@{}}
\toprule
Modalities & SAEs compared & Cosine distance \\
\midrule
Within & Image--image & $0.168$ \\
 & Text--text & $0.170$ \\
Across & Image--text & $0.560$ \\
\bottomrule
\end{tabular}
\end{subtable}\hfill
\begin{subtable}[t]{0.51\linewidth}
\centering
\caption{Cross-modal pairs filtered by stability}
\label{tab:stability_validation}

\small
\setlength{\tabcolsep}{4pt}
\begin{tabular}{lcc}
\toprule
Stability rank & Pairs (\% of all) & Cosine distance \\
\midrule
Top $5\%$ & $4.2$ & $0.271$ \\
Top $10\%$ & $8.4$ & $0.349$ \\
All pairs & $100$ & $0.560$ \\
\bottomrule
\end{tabular}
\end{subtable}
\end{table}

For each modality, we train two SAEs with different random seeds and compare the learned feature directions within and across modalities. For each comparison, we retain feature pairs with activation correlation at least $0.6$ and measure their cosine distance. Table~\ref{tab:directional_validation}(\subref{tab:seed_control_main}) shows that the mean cross-modal distance is larger than the corresponding same-modality distances in this experiment.

Following prior work~\citep{fel2025archetypal,papadimitriou2025interpreting}, we rank features by their stability across independent training runs after matching feature directions with the Hungarian algorithm. Starting from the cross-modal pairs with correlation at least $0.6$, we retain pairs whose image and text features both lie in the top $5\%$ or $10\%$ of their respective stability rankings. Table~\ref{tab:directional_validation}(\subref{tab:stability_validation}) shows that the mean distances remain positive in these subsets.

\subsection{Examples of Cross-Modal Feature Heterogeneity}
\label{app:qualitative_examples}

\input{figures/qualitative_heterogeneity}

Figure~\ref{fig:qualitative_heterogeneity} shows examples where matched image and text latents strongly coactivate and respond to the same concept, such as \emph{baseball} or \emph{policeman}. In the \emph{baseball} case, the image latent with index $i=1789$ and the text latent with index $j=1654$ have a high coactivation correlation, $c_{i,j}=0.67$. But the distance between the estimated features remains large, with $d_{\cos}(\hat{\vphi}_i, \hat{\vpsi}_j) = 0.42$. These examples show that semantic correspondence need not imply directional identity.

\subsection{Cross-Modal Feature Heterogeneity in Larger VLMs}
\label{app:heterogeneity_large_backbones}

\input{figures/multi_density_large}

To check whether the cross-modal feature heterogeneity persists at larger scales, we repeat the same analysis shown in Figure~\ref{fig:multi_density} on the corresponding large-size models (ViT-L/14) of \texttt{CLIP}, \texttt{MetaCLIP}, \texttt{OpenCLIP}, and \texttt{SigLIP2}, and report the results in Figure~\ref{fig:multi_density_large}.
The overall trend is consistent with Figure~\ref{fig:multi_density}: feature pairs with larger coactivation correlation tend to have smaller cosine distances, but they do not concentrate near zero.
This indicates that cross-modal feature heterogeneity persists across embedding models of different sizes.

\subsection{Alternative Post-Hoc Alignments}
\label{app:alignment_validation}

We compare Hungarian matching~\citep{kuhn1955hungarian} with Sinkhorn optimal transport~\citep{cuturi2013sinkhorn}, orthogonal Procrustes alignment~\citep{schonemann1966generalized}, and canonical correlation analysis (CCA). Sinkhorn permits soft correspondences, whereas Procrustes and CCA can mix latent coordinates. All methods in Table~\ref{tab:alternative_alignment} use the same image and text SAEs trained separately on \texttt{CLIP} ViT-B/32 embeddings from \texttt{CC-3M}. Each alignment is fitted using paired \texttt{CC-3M} latent activations and evaluated on \texttt{MS-COCO} retrieval. While Table~\ref{tab:real_sae_compact} reports averages over three independent runs, this comparison uses one run per method.

\begin{table}[h]
    \centering
    \small
    \setlength{\tabcolsep}{5pt}
    \caption{Alternative post-hoc alignments in the additional comparison.}
    \label{tab:alternative_alignment}
    \begin{tabular}{lrrrrrr}
        \toprule
        & \multicolumn{3}{c}{Image-to-text} & \multicolumn{3}{c}{Text-to-image} \\
        \cmidrule(lr){2-4}\cmidrule(lr){5-7}
        Post-hoc Alignment Method & R@1 & R@5 & R@10 & R@1 & R@5 & R@10 \\
        \midrule
        Hungarian (Our main method) & $18.7$ & $37.9$ & $49.1$ & $11.4$ & $27.5$ & $37.2$ \\
        Sinkhorn, $\epsilon=0.01$ & $19.4$ & $40.9$ & $52.2$ & $12.9$ & $30.0$ & $40.3$ \\
        Sinkhorn, $\epsilon=0.05$ & $18.0$ & $36.0$ & $46.9$ & $7.5$ & $19.8$ & $28.6$ \\
        Procrustes & $25.5$ & $48.4$ & $60.0$ & $17.5$ & $37.6$ & $48.4$ \\
        CCA, 256 dimensions & $18.0$ & $36.0$ & $46.8$ & $11.6$ & $26.4$ & $36.1$ \\
        CCA, 1024 dimensions & $20.9$ & $40.7$ & $51.5$ & $12.3$ & $29.1$ & $39.8$ \\
        \bottomrule
    \end{tabular}
\end{table}

Modality-specific SAEs with Procrustes post-hoc alignment achieve the highest retrieval scores in Table~\ref{tab:alternative_alignment}, showing that greater alignment flexibility can benefit retrieval. Hungarian matching preserves explicit correspondences between individual features, following prior SAE stability analyses~\citep{fel2025archetypal,papadimitriou2025interpreting} and dictionary-learning theory~\citep{spielman2012exact}. All evaluated methods align fixed modality-specific SAEs and outperform the shared-SAE retrieval baselines, supporting the separation of feature recovery from alignment. The alignment method can therefore be chosen according to whether retrieval performance or interpretation of individual features is the primary objective.

\subsection{Evaluation Across VLMs and SAE Configurations}
\label{app:additional_real_results}

In the main experiment in Section~\ref{subsec:real_exp}, Table~\ref{tab:real_sae_compact} reports results for a single VLM embedding and activation setup (\texttt{CLIP} ViT-B/32 embeddings with Top-$K$ activation at $K=32$).
To verify whether the conclusions hold across diverse setups, we conduct extended experiments by varying three factors: (i) the sparsity level $K$ in Top-$K$ activation, (ii) the pre-trained VLM backbone (varying both architecture scale and pre-training corpus), and (iii) the choice of sparsifying activation function.
All other settings follow Section~\ref{subsec:real_exp}, and we report averages over three independent runs.
Across these settings, our method preserves the reconstruction error of modality-specific SAEs and outperforms the shared-SAE baselines on retrieval in most cases.

\paragraph{Robustness to sparsity level $K$.}
We first vary the Top-$K$ sparsity level on \texttt{CLIP} ViT-B/32 embeddings.
Specifically, we sweep $K \in \{16, 32, 64\}$ while keeping all other hyperparameters identical to the main experiment in Section~\ref{subsec:real_exp}; the $K=32$ result is reported in Table~\ref{tab:real_sae_compact}, and the remaining levels in Table~\ref{tab:ext_sparsity}.
Across all three levels, our method preserves the lowest reconstruction error inherited from modality-specific SAEs, and $K=32$ emerges as the best operating point for cross-modal retrieval, where our method peaks at Image-to-Text Recall@$1$ of $16.0$.
The gain over the strongest baseline is largest in the sparser regime: at $K=16$, our method more than doubles its performance in both retrieval directions, whereas at $K=64$ the auxiliary-loss baselines benefit from the larger active budget, though our method still leads on Text-to-Image Recall and on top-$10$ retrieval in both directions.
Overall, $K=32$ is the optimal operating point for cross-modal retrieval, and at this level our method consistently outperforms all baselines.

\input{tables/ext_sparsity}

\paragraph{Robustness across vision-language model backbones.}
We next vary the pre-trained VLM backbone while fixing $K=32$ with Top-$K$ activation.
Specifically, we span two architecture scales (ViT-B/32 and ViT-L/14) and two pre-training corpora (\texttt{CLIP}~\citep{radford21clip} and \texttt{OpenCLIP}~\citep{cherti2023reproducible}).
We keep all other hyperparameters identical to the main experiment in Section~\ref{subsec:real_exp}.
Specifically, for each backbone we scale the SAE width $m$ with the embedding dimension $d$ to hold the expansion ratio $m/d$ fixed, so that comparisons are not confounded by SAE capacity.
The results are reported in Table~\ref{tab:ext_backbone}, with the \texttt{CLIP} ViT-B/32 setting in Table~\ref{tab:real_sae_compact}.
Across all backbones, our method retains the lowest reconstruction error while attaining the strongest cross-modal retrieval, confirming that the effect is not specific to a single architecture or pre-training corpus.
Therefore, these results show that the advantage of our method holds across different VLM backbones.

\input{tables/ext_backbone}

\paragraph{Robustness to activation function.}
We replace the sparsifying activation function, changing from Top-$K$ to BatchTop-$K$~\citep{bussmann2024batchtopk}, which selects the top-$K$ activations across the entire batch rather than per sample.
We keep all other hyperparameters identical to the main experiment in Section~\ref{subsec:real_exp}.
The results are reported in Table~\ref{tab:ext_activation}.
Our method again attains the lowest reconstruction error and the strongest cross-modal retrieval.
This confirms that the advantage of our method is robust to the choice of sparsifying activation function.

\input{tables/ext_activation}

\subsection{Monosemanticity Analysis}
\label{app:monosemanticity_detail}

We examine whether the learned latent coordinates are activated by semantically coherent inputs. We use the monosemanticity score~\citep{pach2026sparse} on the validation set of \texttt{CC-3M}. This score measures whether inputs that activate the same latent are close to one another in an external embedding space. We use the \texttt{MetaCLIP} model~\citep{chuang2026meta} as the external encoder. Larger values indicate that the latent responds to a more semantically coherent concept.

\input{figures/monosemanticity_score}

Figure~\ref{fig:monosemanticity_score} reports monosemanticity scores for each latent coordinate, sorted in descending order. Our method maintains higher scores across more latent coordinates in both modalities. This indicates that preserving modality-specific feature directions gives more semantically coherent latents.

%% file: figures/qualitative_heterogeneity.tex
\begin{figure*}[h]
  \centering
  \includegraphics[width=\textwidth]{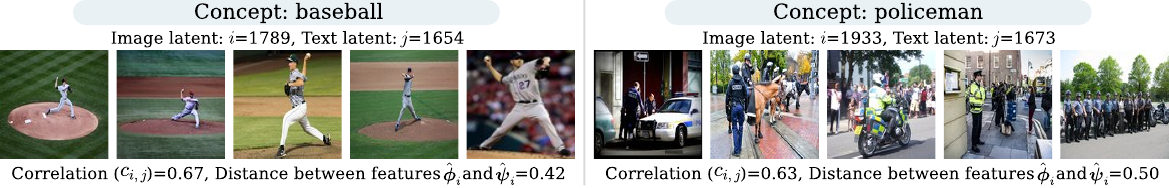}
  \caption{
  Image examples with high coactivation correlation $c_{i,j}$ from the \texttt{CLIP} setting in Figure~\ref{fig:multi_density}, where $i$ and $j$ denote the image and text latent indices, respectively. The images reveal coherent concepts such as \emph{baseball} and \emph{policeman}, but the features $\hat{\vphi}_i$ and $\hat{\vpsi}_j$ remain directionally separated.
  }
  \label{fig:qualitative_heterogeneity}
\end{figure*}

%% file: figures/multi_density_large.tex
\begin{figure*}[h]
  \centering
  \includegraphics[width=1.0\textwidth]{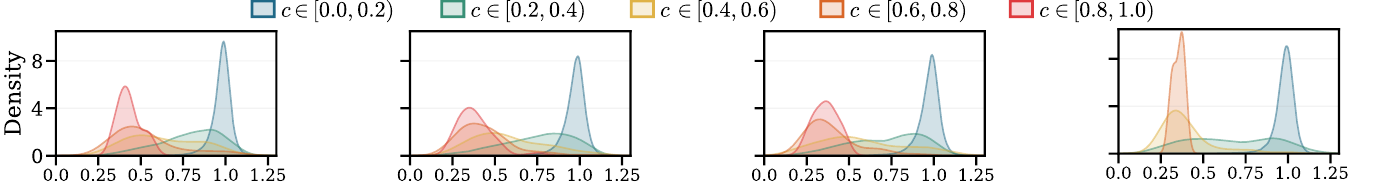}
  \resizebox{\textwidth}{!}{%
    \begin{tabular}{*{4}{p{0.25\textwidth}}}
        (a) \makecell{\texttt{CLIP}~\citep{radford21clip} ViT-L/14} &
        (b) \makecell{\texttt{MetaCLIP}~\citep{chuang2026meta} L/14} &
        (c) \makecell{\texttt{OpenCLIP}~\citep{cherti2023reproducible} L/14} &
        (d) \makecell{\texttt{SigLIP2}~\citep{tschannen2025siglip} Large}
    \end{tabular}%
}
  \caption{
    Distribution of cosine distances between image--text feature pairs estimated from embeddings of four VLMs.
    This figure presents the same analysis as Figure~\ref{fig:multi_density}, extended to larger models.
  }
  \label{fig:multi_density_large}
\end{figure*}

%% file: tables/ext_sparsity.tex
\begin{table}[h]
    \centering
    \caption{Effect of Top-$K$ sparsity level $K$ on CLIP ViT-B/32 embeddings. Each section corresponds to a different value of $K$. Best and second-best results are shown in bold and underlined, respectively.}
    \label{tab:ext_sparsity}
    {
    \setlength{\tabcolsep}{2pt}
    \resizebox{\linewidth}{!}{
    \begin{tabular}{l c ccc ccc c c}
        \toprule
        & \multicolumn{7}{c}{MS-COCO~\citep{lin2014microsoft}} & \multicolumn{2}{c}{ImageNet1K~\citep{deng2009imagenet}} \\
        \cmidrule(lr){2-8} \cmidrule(lr){9-10}
        & Recon. ($\downarrow$)
        & \multicolumn{3}{c}{Image-to-Text ($\uparrow$)}
        & \multicolumn{3}{c}{Text-to-Image ($\uparrow$)}
        & Recon. ($\downarrow$) & Zero-shot ($\uparrow$) \\
        \cmidrule(lr){2-2} \cmidrule(lr){3-5} \cmidrule(lr){6-8} \cmidrule(lr){9-9} \cmidrule(lr){10-10}
        Methods & MSE
        & R@1 & R@5 & R@10
        & R@1 & R@5 & R@10
        & MSE & Accuracy \\
        \midrule
        \multicolumn{10}{c}{\textit{$K = 16$}} \\
        \midrule
        Shared SAE
        & 0.131 
        & 3.0 {\tiny $(\pm 1.9)$} & 7.7 {\tiny $(\pm 3.3)$} & 10.9 {\tiny $(\pm 4.1)$}
        & 3.0 {\tiny $(\pm 0.3)$} & 7.8 {\tiny $(\pm 0.7)$} & 11.3 {\tiny $(\pm 1.1)$}
        & 0.169  & 12.6 {\tiny $(\pm 4.1)$} \\
        + Iso-Energy alignment loss
        & 0.131
        & 4.2 {\tiny $(\pm 1.2)$} & 9.6 {\tiny $(\pm 2.8)$} & 13.6 {\tiny $(\pm 3.8)$}
        & \underline{3.1} {\tiny $(\pm 0.2)$} & 8.2 {\tiny $(\pm 0.8)$} & 12.1 {\tiny $(\pm 1.3)$}
        & 0.168  & 15.3 {\tiny $(\pm 3.8)$} \\
        + Group-sparse loss
        & 0.134 
        & \underline{5.1} {\tiny $(\pm 0.1)$} & \underline{12.2} {\tiny $(\pm 0.2)$} & \underline{17.2} {\tiny $(\pm 0.2)$}
        & 2.9 {\tiny $(\pm 0.1)$} & \underline{8.8} {\tiny $(\pm 0.5)$} & \underline{12.9} {\tiny $(\pm 0.5)$}
        & 0.170  & \underline{20.3} {\tiny $(\pm 0.6)$} \\
        Modality-Specific SAEs
        & \textbf{0.127} 
        & 0.0 {\tiny $(\pm 0.0)$} & 0.1 {\tiny $(\pm 0.1)$} & 0.2 {\tiny $(\pm 0.1)$}
        & 0.0 {\tiny $(\pm 0.0)$} & 0.1 {\tiny $(\pm 0.1)$} & 0.2 {\tiny $(\pm 0.1)$}
        & \textbf{0.165}  & 0.1 {\tiny $(\pm 0.1)$} \\
        \rowcolor{gray!20}
        + Post-hoc Alignment (Ours)
        & \textbf{0.127} 
        & \textbf{13.3} {\tiny $(\pm 0.7)$} & \textbf{29.4} {\tiny $(\pm 0.3)$} & \textbf{38.9} {\tiny $(\pm 0.7)$}
        & \textbf{9.6} {\tiny $(\pm 0.3)$} & \textbf{23.0} {\tiny $(\pm 0.5)$} & \textbf{31.8} {\tiny $(\pm 0.5)$}
        & \textbf{0.165}  & \textbf{22.9} {\tiny $(\pm 0.5)$} \\
        \midrule
        \multicolumn{10}{c}{\textit{$K = 64$}} \\
        \midrule
        Shared SAE
        & 0.063 
        & 9.8 {\tiny $(\pm 1.3)$} & 21.2 {\tiny $(\pm 2.0)$} & 28.0 {\tiny $(\pm 2.9)$}
        & 6.0 {\tiny $(\pm 0.6)$} & 15.2 {\tiny $(\pm 1.2)$} & 21.3 {\tiny $(\pm 1.5)$}
        & \textbf{0.083}  & 15.0 {\tiny $(\pm 1.5)$} \\
        + Iso-Energy alignment loss
        & 0.063 
        & 9.0 {\tiny $(\pm 0.8)$} & 19.9 {\tiny $(\pm 0.4)$} & 26.7 {\tiny $(\pm 0.5)$}
        & 5.3 {\tiny $(\pm 0.6)$} & 13.4 {\tiny $(\pm 1.5)$} & 18.8 {\tiny $(\pm 2.0)$}
        & 0.084  & 16.3 {\tiny $(\pm 1.2)$} \\
        + Group-sparse loss
        & 0.091 
        & \textbf{11.8} {\tiny $(\pm 0.5)$} & \textbf{26.0} {\tiny $(\pm 0.8)$} & \underline{34.6} {\tiny $(\pm 1.1)$}
        & \underline{7.0} {\tiny $(\pm 0.3)$} & \underline{18.4} {\tiny $(\pm 0.5)$} & \underline{26.4} {\tiny $(\pm 0.5)$}
        & 0.116  & \textbf{33.0} {\tiny $(\pm 0.4)$} \\
        Modality-Specific SAEs
        & \textbf{0.062} 
        & 0.0 {\tiny $(\pm 0.0)$} & 0.1 {\tiny $(\pm 0.0)$} & 0.2 {\tiny $(\pm 0.1)$}
        & 0.0 {\tiny $(\pm 0.0)$} & 0.1 {\tiny $(\pm 0.0)$} & 0.3 {\tiny $(\pm 0.1)$}
        & \textbf{0.083}  & 0.1 {\tiny $(\pm 0.0)$} \\
        \rowcolor{gray!20}
        + Post-hoc Alignment (Ours)
        & \textbf{0.062} 
        & \underline{11.2} {\tiny $(\pm 0.8)$} & \underline{26.0} {\tiny $(\pm 1.6)$} & \textbf{34.9} {\tiny $(\pm 2.2)$}
        & \textbf{8.5} {\tiny $(\pm 0.8)$} & \textbf{21.5} {\tiny $(\pm 2.1)$} & \textbf{29.7} {\tiny $(\pm 2.8)$}
        & \textbf{0.083}  & \underline{19.3} {\tiny $(\pm 1.2)$} \\
        \bottomrule
    \end{tabular}
    }
    }
\end{table}

%% file: tables/ext_backbone.tex
\begin{table}[h]
    \centering
    \caption{Effect of the size and type of the pre-trained VLM backbone.
    Each section corresponds to a different backbone, with the SAE latent dimension $m$ shown in the header.
    }
    \label{tab:ext_backbone}
    {
    \setlength{\tabcolsep}{2pt}
    \resizebox{\linewidth}{!}{
    \begin{tabular}{l c ccc ccc c c}
        \toprule
        & \multicolumn{7}{c}{MS-COCO~\citep{lin2014microsoft}} & \multicolumn{2}{c}{ImageNet1K~\citep{deng2009imagenet}} \\
        \cmidrule(lr){2-8} \cmidrule(lr){9-10}
        & Recon. ($\downarrow$)
        & \multicolumn{3}{c}{Image-to-Text ($\uparrow$)}
        & \multicolumn{3}{c}{Text-to-Image ($\uparrow$)}
        & Recon. ($\downarrow$) & Zero-shot ($\uparrow$) \\
        \cmidrule(lr){2-2} \cmidrule(lr){3-5} \cmidrule(lr){6-8} \cmidrule(lr){9-9} \cmidrule(lr){10-10}
        Methods & MSE
        & R@1 & R@5 & R@10
        & R@1 & R@5 & R@10
        & MSE & Accuracy \\
        \midrule
        \multicolumn{10}{c}{\textit{CLIP ViT-L/14 ($m = 12288$)}} \\
        \midrule
        Shared SAE
        & 0.119 
        & \underline{18.2} {\tiny $(\pm 0.2)$} & \underline{33.6} {\tiny $(\pm 0.4)$} & \underline{41.8} {\tiny $(\pm 0.7)$}
        & \underline{9.9} {\tiny $(\pm 0.8)$} & \underline{22.9} {\tiny $(\pm 0.9)$} & \underline{30.7} {\tiny $(\pm 1.2)$}
        & 0.155  & 38.9 {\tiny $(\pm 0.6)$} \\
        + Iso-Energy alignment loss
        & 0.119 
        & 16.8 {\tiny $(\pm 1.1)$} & 31.2 {\tiny $(\pm 1.2)$} & 39.3 {\tiny $(\pm 1.7)$}
        & 9.4 {\tiny $(\pm 0.5)$} & 21.3 {\tiny $(\pm 0.9)$} & 28.9 {\tiny $(\pm 1.2)$}
        & \textbf{0.154}  & \underline{39.4} {\tiny $(\pm 0.8)$} \\
        + Group-sparse loss
        & 0.137 
        & 14.2 {\tiny $(\pm 0.2)$} & 29.6 {\tiny $(\pm 0.7)$} & 38.5 {\tiny $(\pm 1.0)$}
        & 8.2 {\tiny $(\pm 0.4)$} & 20.3 {\tiny $(\pm 0.4)$} & 28.4 {\tiny $(\pm 0.3)$}
        & 0.176  & \textbf{45.0} {\tiny $(\pm 0.3)$} \\
        Modality-Specific SAEs
        & \textbf{0.117} 
        & 0.0 {\tiny $(\pm 0.0)$} & 0.1 {\tiny $(\pm 0.0)$} & 0.2 {\tiny $(\pm 0.0)$}
        & 0.0 {\tiny $(\pm 0.0)$} & 0.1 {\tiny $(\pm 0.0)$} & 0.2 {\tiny $(\pm 0.1)$}
        & \textbf{0.154}  & 0.1 {\tiny $(\pm 0.0)$} \\
        \rowcolor{gray!20}
        + Post-hoc Alignment (Ours)
        & \textbf{0.117} 
        & \textbf{22.6} {\tiny $(\pm 0.7)$} & \textbf{43.6} {\tiny $(\pm 1.2)$} & \textbf{54.9} {\tiny $(\pm 1.2)$}
        & \textbf{15.3} {\tiny $(\pm 0.4)$} & \textbf{33.0} {\tiny $(\pm 0.4)$} & \textbf{42.9} {\tiny $(\pm 0.3)$}
        & \textbf{0.154}  & 35.9 {\tiny $(\pm 1.1)$} \\
        \midrule
        \multicolumn{10}{c}{\textit{OpenCLIP ViT-B/32 ($m = 8192$)}} \\
        \midrule
        Shared SAE
        & \underline{0.116} 
        & 8.9 {\tiny $(\pm 0.3)$} & 18.5 {\tiny $(\pm 0.3)$} & 24.4 {\tiny $(\pm 0.5)$}
        & 4.8 {\tiny $(\pm 0.2)$} & 12.2 {\tiny $(\pm 0.2)$} & 17.4 {\tiny $(\pm 0.2)$}
        & \textbf{0.149}  & 18.7 {\tiny $(\pm 0.7)$} \\
        + Iso-Energy alignment loss
        & 0.117 
        & 9.2 {\tiny $(\pm 0.8)$} & 18.6 {\tiny $(\pm 1.3)$} & 24.5 {\tiny $(\pm 1.2)$}
        & 4.8 {\tiny $(\pm 0.5)$} & 12.3 {\tiny $(\pm 0.8)$} & 17.4 {\tiny $(\pm 1.1)$}
        & \textbf{0.149}  & 18.9 {\tiny $(\pm 1.3)$} \\
        + Group-sparse loss
        & 0.141 
        & \underline{10.1} {\tiny $(\pm 0.4)$} & \underline{23.1} {\tiny $(\pm 0.4)$} & \underline{31.8} {\tiny $(\pm 0.5)$}
        & \underline{6.1} {\tiny $(\pm 0.0)$} & \underline{17.1} {\tiny $(\pm 0.2)$} & \underline{25.4} {\tiny $(\pm 0.3)$}
        & 0.176  & \textbf{34.3} {\tiny $(\pm 0.4)$} \\
        Modality-Specific SAEs
        & \textbf{0.115} 
        & 0.0 {\tiny $(\pm 0.0)$} & 0.1 {\tiny $(\pm 0.0)$} & 0.2 {\tiny $(\pm 0.1)$}
        & 0.0 {\tiny $(\pm 0.0)$} & 0.1 {\tiny $(\pm 0.0)$} & 0.2 {\tiny $(\pm 0.1)$}
        & \textbf{0.149}  & 0.1 {\tiny $(\pm 0.0)$} \\
        \rowcolor{gray!20}
        + Post-hoc Alignment (Ours)
        & \textbf{0.115} 
        & \textbf{21.0} {\tiny $(\pm 0.9)$} & \textbf{41.2} {\tiny $(\pm 1.1)$} & \textbf{52.0} {\tiny $(\pm 1.0)$}
        & \textbf{11.2} {\tiny $(\pm 0.1)$} & \textbf{26.5} {\tiny $(\pm 0.6)$} & \textbf{36.1} {\tiny $(\pm 0.9)$}
        & \textbf{0.149}  & \underline{29.2} {\tiny $(\pm 0.7)$} \\
        \midrule
        \multicolumn{10}{c}{\textit{OpenCLIP ViT-L/14 ($m = 12288$)}} \\
        \midrule
        Shared SAE
        & \underline{0.133} 
        & 26.5 {\tiny $(\pm 0.4)$} & 45.5 {\tiny $(\pm 1.8)$} & 54.5 {\tiny $(\pm 2.1)$}
        & 15.8 {\tiny $(\pm 0.8)$} & 32.0 {\tiny $(\pm 1.5)$} & 40.6 {\tiny $(\pm 1.7)$}
        & \textbf{0.164}  & \underline{50.9} {\tiny $(\pm 0.7)$} \\
        + Iso-Energy alignment loss
        & 0.134 
        & \underline{27.1} {\tiny $(\pm 0.4)$} & \underline{46.2} {\tiny $(\pm 1.1)$} & \underline{55.5} {\tiny $(\pm 0.7)$}
        & \underline{16.2} {\tiny $(\pm 0.5)$} & \underline{32.9} {\tiny $(\pm 0.9)$} & \underline{41.6} {\tiny $(\pm 0.8)$}
        & 0.166  & 50.3 {\tiny $(\pm 0.5)$} \\
        + Group-sparse loss
        & 0.162 
        & 23.7 {\tiny $(\pm 0.2)$} & 43.3 {\tiny $(\pm 0.5)$} & 53.9 {\tiny $(\pm 0.1)$}
        & 13.1 {\tiny $(\pm 0.3)$} & 29.1 {\tiny $(\pm 0.4)$} & 38.6 {\tiny $(\pm 0.5)$}
        & 0.196  & \textbf{53.9} {\tiny $(\pm 0.3)$} \\
        Modality-Specific SAEs
        & \textbf{0.129} 
        & 0.1 {\tiny $(\pm 0.0)$} & 0.1 {\tiny $(\pm 0.1)$} & 0.2 {\tiny $(\pm 0.1)$}
        & 0.0 {\tiny $(\pm 0.0)$} & 0.1 {\tiny $(\pm 0.1)$} & 0.2 {\tiny $(\pm 0.1)$}
        & \textbf{0.164}  & 0.1 {\tiny $(\pm 0.0)$} \\
        \rowcolor{gray!20}
        + Post-hoc Alignment (Ours)
        & \textbf{0.129} 
        & \textbf{27.4} {\tiny $(\pm 2.2)$} & \textbf{50.6} {\tiny $(\pm 1.8)$} & \textbf{62.0} {\tiny $(\pm 1.4)$}
        & \textbf{18.1} {\tiny $(\pm 2.0)$} & \textbf{37.6} {\tiny $(\pm 3.0)$} & \textbf{47.9} {\tiny $(\pm 3.1)$}
        & \textbf{0.164}  & 42.4 {\tiny $(\pm 0.4)$} \\
        \bottomrule
    \end{tabular}
    }
    }
\end{table}

%% file: tables/ext_activation.tex
\begin{table}[h]
    \centering
    \caption{Effect of sparsifying activation function on CLIP ViT-B/32 ($m=8192$, $K=32$). The Top-$K$ setting is reported in Table~\ref{tab:real_sae_compact}; this table reports the BatchTop-$K$~\citep{bussmann2024batchtopk} variant. Best and second-best results are shown in bold and underlined, respectively.}
    \label{tab:ext_activation}
    {
    \setlength{\tabcolsep}{2pt}
    \resizebox{\linewidth}{!}{
    \begin{tabular}{l c ccc ccc c c}
        \toprule
        & \multicolumn{7}{c}{MS-COCO~\citep{lin2014microsoft}} & \multicolumn{2}{c}{ImageNet1K~\citep{deng2009imagenet}} \\
        \cmidrule(lr){2-8} \cmidrule(lr){9-10}
        & Recon. ($\downarrow$)
        & \multicolumn{3}{c}{Image-to-Text ($\uparrow$)}
        & \multicolumn{3}{c}{Text-to-Image ($\uparrow$)}
        & Recon. ($\downarrow$) & Zero-shot ($\uparrow$) \\
        \cmidrule(lr){2-2} \cmidrule(lr){3-5} \cmidrule(lr){6-8} \cmidrule(lr){9-9} \cmidrule(lr){10-10}
        Methods & MSE
        & R@1 & R@5 & R@10
        & R@1 & R@5 & R@10
        & MSE & Accuracy \\
        \midrule
        \multicolumn{10}{c}{\textit{BatchTop-$K$}} \\
        \midrule
        Shared SAE
        & 0.089 {\tiny }
        & 6.3 {\tiny $(\pm 0.6)$} & 15.9 {\tiny $(\pm 1.5)$} & 22.0 {\tiny $(\pm 1.7)$}
        & 3.9 {\tiny $(\pm 0.6)$} & 11.3 {\tiny $(\pm 0.9)$} & 16.8 {\tiny $(\pm 1.0)$}
        & 0.115 {\tiny } & 16.4 {\tiny $(\pm 3.1)$} \\
        + Iso-Energy alignment loss
        & 0.089 {\tiny }
        & 5.8 {\tiny $(\pm 0.8)$} & 14.8 {\tiny $(\pm 0.8)$} & 20.4 {\tiny $(\pm 0.9)$}
        & 3.3 {\tiny $(\pm 1.4)$} & 9.3 {\tiny $(\pm 3.3)$} & 13.9 {\tiny $(\pm 3.9)$}
        & 0.115 {\tiny } & 13.0 {\tiny $(\pm 1.8)$} \\
        + Group-sparse loss
        & 0.105 {\tiny }
        & \underline{7.8} {\tiny $(\pm 0.3)$} & \underline{19.0} {\tiny $(\pm 0.4)$} & \underline{26.2} {\tiny $(\pm 0.4)$}
        & \underline{4.9} {\tiny $(\pm 0.0)$} & \underline{13.7} {\tiny $(\pm 0.1)$} & \underline{20.6} {\tiny $(\pm 0.1)$}
        & 0.132 {\tiny } & \textbf{28.0} {\tiny $(\pm 0.7)$} \\
        Modality-Specific SAEs
        & \textbf{0.088} {\tiny }
        & 0.0 {\tiny $(\pm 0.0)$} & 0.0 {\tiny $(\pm 0.0)$} & 0.1 {\tiny $(\pm 0.0)$}
        & 0.0 {\tiny $(\pm 0.0)$} & 0.1 {\tiny $(\pm 0.0)$} & 0.2 {\tiny $(\pm 0.0)$}
        & \textbf{0.114} {\tiny } & 0.1 {\tiny $(\pm 0.0)$} \\
        \rowcolor{gray!20}
        + Post-hoc Alignment (Ours)
        & \textbf{0.088} {\tiny }
        & \textbf{16.2} {\tiny $(\pm 0.8)$} & \textbf{34.6} {\tiny $(\pm 1.0)$} & \textbf{44.7} {\tiny $(\pm 1.4)$}
        & \textbf{10.5} {\tiny $(\pm 0.6)$} & \textbf{25.5} {\tiny $(\pm 0.8)$} & \textbf{35.0} {\tiny $(\pm 0.9)$}
        & \textbf{0.114} {\tiny } & \underline{22.8} {\tiny $(\pm 1.0)$} \\
        \bottomrule
    \end{tabular}
    }
    }
\end{table}

%% file: figures/monosemanticity_score.tex
\begin{figure*}[h]
    \centering
    \includegraphics[width=0.65\textwidth]{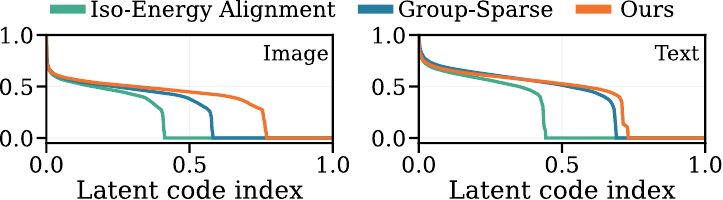}
    \caption{
    Monosemanticity scores for coordinates sorted in descending order, shown separately for image and text modalities.
    }
    \label{fig:monosemanticity_score}
\end{figure*}

%% file: texts/checklist.tex
\section*{NeurIPS Paper Checklist}

\begin{enumerate}

\item {\bf Claims}
    \item[] Question: Do the main claims made in the abstract and introduction accurately reflect the paper's contributions and scope?
    \item[] Answer: \answerYes{}.
    \item[] Justification: The abstract and introduction state the main claims and contributions, and these claims are supported by the theory in Section~\ref{sec:theory} and experiments in Section~\ref{sec:experiments}. The assumptions and scope are discussed in Section~\ref{sec:prelim} and Appendix~\ref{sec:limitation}.

\item {\bf Limitations}
    \item[] Question: Does the paper discuss the limitations of the work performed by the authors?
    \item[] Answer: \answerYes{}
    \item[] Justification: Limitations are discussed in Appendix~\ref{sec:limitation}.

\item {\bf Theory assumptions and proofs}
    \item[] Question: For each theoretical result, does the paper provide the full set of assumptions and a complete (and correct) proof?
    \item[] Answer: \answerYes{}
    \item[] Justification: Assumptions are stated in Section~\ref{sec:prelim} and Section~\ref{sec:theory}. Proofs are provided in Appendix~\ref{app:heterogeneity} and Appendix~\ref{sec:proofs}.

    \item {\bf Experimental result reproducibility}
    \item[] Question: Does the paper fully disclose all the information needed to reproduce the main experimental results of the paper to the extent that it affects the main claims and/or conclusions of the paper (regardless of whether the code and data are provided or not)?
    \item[] Answer: \answerYes{}
    \item[] Justification: Experimental details and code availability are provided in Section~\ref{sec:experiments} and Appendix~\ref{sec:exp_details}.

\item {\bf Open access to data and code}
    \item[] Question: Does the paper provide open access to the data and code, with sufficient instructions to faithfully reproduce the main experimental results, as described in supplemental material?
    \item[] Answer: \answerYes{}
    \item[] Justification: Code and reproduction instructions are available at \url{https://github.com/JiH00nKw0n/cross_modal_feature_heterogeneity}, and the experiments use publicly available datasets.

\item {\bf Experimental setting/details}
    \item[] Question: Does the paper specify all the training and test details (e.g., data splits, hyperparameters, how they were chosen, type of optimizer) necessary to understand the results?
    \item[] Answer: \answerYes{}
    \item[] Justification: Training and evaluation details are provided in Section~\ref{sec:experiments} and Appendix~\ref{sec:exp_details}.


\item {\bf Experiment statistical significance}
    \item[] Question: Does the paper report error bars suitably and correctly defined or other appropriate information about the statistical significance of the experiments?
    \item[] Answer: \answerYes{}
    \item[] Justification: Results are averaged over three independent runs, as stated in Section~\ref{sec:experiments} and Appendix~\ref{app:extended_experiments}.

\item {\bf Experiments compute resources}
    \item[] Question: For each experiment, does the paper provide sufficient information on the computer resources (type of compute workers, memory, time of execution) needed to reproduce the experiments?
    \item[] Answer: \answerYes{}
    \item[] Justification: Compute resources are reported in Appendix~\ref{sec:exp_details}.
    
\item {\bf Code of ethics}
    \item[] Question: Does the research conducted in the paper conform, in every respect, with the NeurIPS Code of Ethics \url{https://neurips.cc/public/EthicsGuidelines}?
    \item[] Answer: \answerYes{}
    \item[] Justification: The work uses public datasets and does not involve human subjects or private data.

\item {\bf Broader impacts}
    \item[] Question: Does the paper discuss both potential positive societal impacts and negative societal impacts of the work performed?
    \item[] Answer: \answerNA{}
    \item[] Justification: This work is foundational research on interpreting and aligning features in joint embedding spaces. We do not identify a direct path to negative applications, including privacy, security, fairness, or misuse-related harms.
    
\item {\bf Safeguards}
    \item[] Question: Does the paper describe safeguards that have been put in place for responsible release of data or models that have a high risk for misuse (e.g., pre-trained language models, image generators, or scraped datasets)?
    \item[] Answer: \answerNA{}
    \item[] Justification: The paper does not release data or models with a high risk for misuse.

\item {\bf Licenses for existing assets}
    \item[] Question: Are the creators or original owners of assets (e.g., code, data, models), used in the paper, properly credited and are the license and terms of use explicitly mentioned and properly respected?
    \item[] Answer: \answerYes{}
    \item[] Justification: The paper cites the original papers for the assets used, and discusses licenses and terms of use in Appendix~\ref{sec:exp_details}.

\item {\bf New assets}
    \item[] Question: Are new assets introduced in the paper well documented and is the documentation provided alongside the assets?
    \item[] Answer: \answerYes{}
    \item[] Justification: The released code is available at \url{https://github.com/JiH00nKw0n/cross_modal_feature_heterogeneity}, with experimental details documented in Appendix~\ref{sec:exp_details}.

\item {\bf Crowdsourcing and research with human subjects}
    \item[] Question: For crowdsourcing experiments and research with human subjects, does the paper include the full text of instructions given to participants and screenshots, if applicable, as well as details about compensation (if any)? 
    \item[] Answer: \answerNA{}
    \item[] Justification: The paper does not involve crowdsourcing nor research with human subjects.

\item {\bf Institutional review board (IRB) approvals or equivalent for research with human subjects}
    \item[] Question: Does the paper describe potential risks incurred by study participants, whether such risks were disclosed to the subjects, and whether Institutional Review Board (IRB) approvals (or an equivalent approval/review based on the requirements of your country or institution) were obtained?
    \item[] Answer: \answerNA{}
    \item[] Justification: The paper does not involve crowdsourcing nor research with human subjects.

\item {\bf Declaration of LLM usage}
    \item[] Question: Does the paper describe the usage of LLMs if it is an important, original, or non-standard component of the core methods in this research? Note that if the LLM is used only for writing, editing, or formatting purposes and does \emph{not} impact the core methodology, scientific rigor, or originality of the research, declaration is not required.
    \item[] Answer: \answerNA{}
    \item[] Justification: The core method development does not involve LLMs as important components.

\end{enumerate}